\documentclass[11pt]{article}

\usepackage[preprint]{acl}

\usepackage{times}
\usepackage{latexsym}

\usepackage[T1]{fontenc}

\usepackage[utf8]{inputenc}

\usepackage{microtype}

\usepackage{inconsolata}

\usepackage{graphicx}
\usepackage{amsmath}
\usepackage{amssymb}
\usepackage{enumitem}
\usepackage{booktabs, multirow, multicol, array, bm}
\usepackage[table]{xcolor} \usepackage{booktabs} \usepackage{multirow}
\usepackage{amsthm}
\newtheorem{assumption}{Assumption}
\newtheorem{lemma}{Lemma}
\newtheorem{theorem}{Theorem}
\newtheorem{corollary}{Corollary}

\title{APEX: Learning Adaptive Priorities for Multi-Objective Alignment in Vision-Language Generation}

\author{
  Dongliang Chen\thanks{\hspace{0.5mm}Equal contribution.}$^{,1}$ \quad
  Xinlin Zhuang$^{*,1}$ \quad
  Junjie Xu$^{*,1}$ \quad
  Luojian Xie$^{1}$ \quad
  Zehui Wang$^{1}$ \\
  \textbf{Jiaxi Zhuang$^{1}$ \quad
  Haolin Yang$^{2}$ \quad
  Liang Dou$^{1}$ \quad
  Xiao He$^{1}$ \quad
  Xingjiao Wu\thanks{\hspace{0.5mm}Corresponding authors.}$^{,1}$ \quad
  Ying Qian$^{\dagger,1}$} \\
  \\
  $^{1}$East China Normal University \quad
  $^{2}$MBZUAI \\
}

\begin{document}

\maketitle

\begin{abstract}

Multi-objective alignment for text-to-image generation is commonly implemented via static linear scalarization, but \textit{fixed} weights often fail under heterogeneous rewards, leading to optimization imbalance where models overfit high-variance, high-responsiveness objectives (e.g., OCR) while under-optimizing perceptual goals. 
We identify two mechanistic causes: \textbf{variance hijacking}, where reward dispersion induces implicit reweighting that dominates the normalized training signal, and \textbf{gradient conflicts}, where competing objectives produce opposing update directions and trigger seesaw-like oscillations. 
We propose \textbf{APEX} (\textbf{A}daptive \textbf{P}riority-based \textbf{E}fficient \textbf{X}-objective Alignment), which stabilizes heterogeneous rewards with \textbf{Dual-Stage Adaptive Normalization} and dynamically schedules objectives via \textbf{\bm{$\mathcal{P}^3$} Adaptive Priorities} that combine learning potential, conflict penalty, and progress need. 
On Stable Diffusion 3.5, APEX achieves improved Pareto trade-offs across four heterogeneous objectives, with balanced gains of \textbf{+1.31 PickScore}, \textbf{+0.35 DeQA}, and \textbf{+0.53 Aesthetics} while maintaining competitive OCR accuracy, mitigating the instability of multi-objective alignment.

\end{abstract}

\section{Introduction}

Vision-language generation \citep{bie2023renaissancesurveyaitexttoimage} has advanced rapidly in recent years, enabling text-to-image (T2I) models based on diffusion and flow matching to synthesize high-fidelity images from natural language prompts \citep{ho2020denoisingdiffusionprobabilisticmodels,Rombach_2022_CVPR,lipman2023flowmatchinggenerativemodeling,bie2023renaissancesurveyaitexttoimage}. 
As these systems move from demos to real use cases, a single notion of \textbf{quality} is no longer sufficient. 
Practical alignment must simultaneously satisfy \textbf{heterogeneous objectives}, ranging from discrete structural constraints (e.g., rendering legible text) to perceptual preferences (aesthetics, realism, artifact suppression) \citep{NEURIPS2023_33646ef0,10655849}. 
However, these objectives are often competing: improving local sharpness to boost text readability can degrade global lighting coherence or introduce artifacts, yielding a brittle ``one-metric-at-a-time'' behavior (as shown in Figure~\ref{fig:teaser}). 
Therefore, achieving \textbf{multi-objective alignment} in T2I generation is  a critical challenge.

\begin{figure}[t]  
  \centering
  \includegraphics[width=\columnwidth]{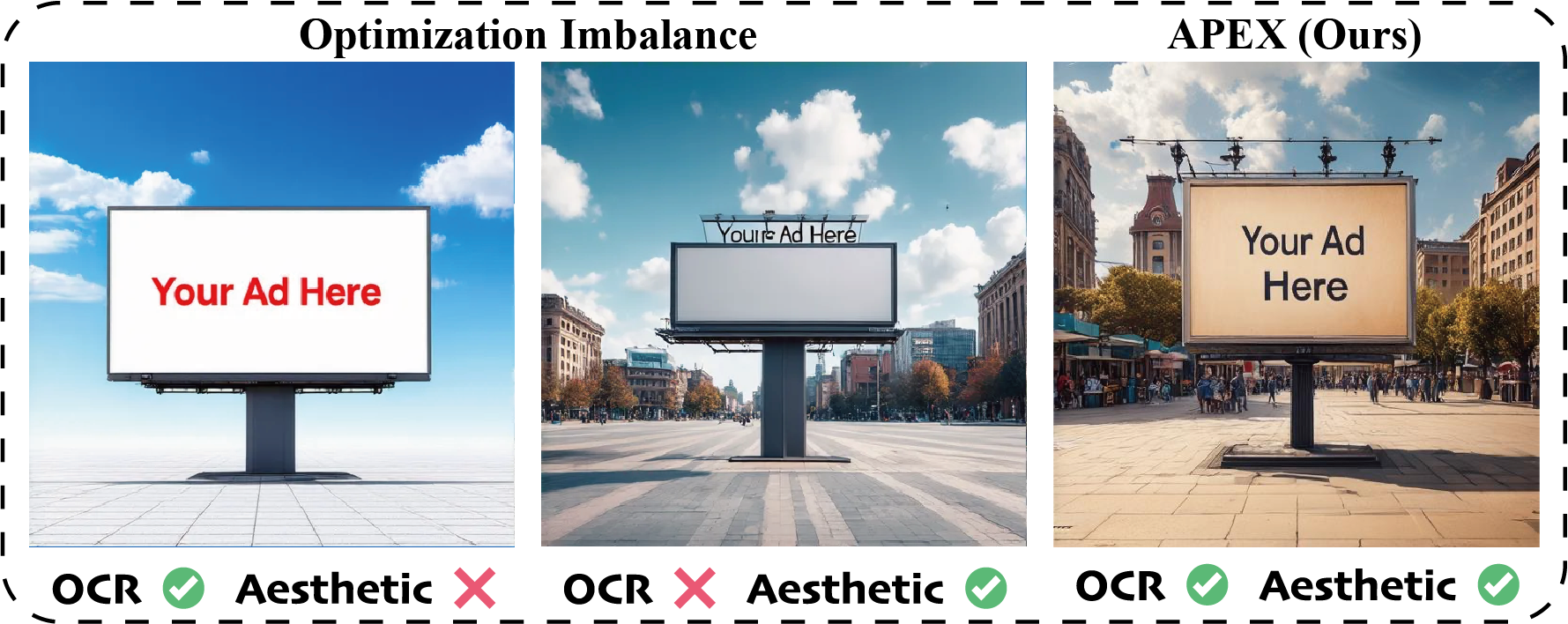} 
  \caption{Images generated by Stable Diffusion 3.5 under different competing objective settings. For the prompt \textit{A city square with a billboard... filled with 'Your Ad Here'}, optimizing for text clarity (\textbf{left}) or visual quality (\textbf{middle}) leads to imbalance. APEX achieves effective multi-objective alignment (\textbf{right}).}
  \label{fig:teaser}
\end{figure}

To optimize black-box, non-differentiable objectives (e.g., OCR metrics and preference models), reinforcement learning (RL) has become a standard paradigm for post-training alignment \citep{ziegler2020finetuninglanguagemodelshuman,black2024trainingdiffusionmodelsreinforcement} 
In the multi-objective setting, existing approaches typically rely on \textbf{Static Linear Scalarization}, which combines disparate reward signals using fixed weights that remain constant throughout training \citep{clark2024directlyfinetuningdiffusionmodels}.
Despite its simplicity, we find that static scalarization fails systematically for heterogeneous rewards, producing severe optimization imbalance that cannot be resolved by merely \textit{tuning weights}. 
Our analysis identifies two mechanistic failure modes that explain this behavior.
\textbf{(1) Variance Hijacking.} 
Even with equal preset weights (Figure~\ref{fig:main}, top-right), objectives with larger dispersion or stronger responsiveness can implicitly dominate the normalized training signal, effectively hijacking the gradient budget. 
In practice, high-variance, discrete constraints such as OCR can saturate early while continuing to monopolize updates, starving low-variance perceptual objectives and preventing further improvement elsewhere. 
This makes the intended scalarization weights unreliable as a control mechanism.
\textbf{(2) Gradient Conflicts.} 
Because objectives share parameters, their policy gradients can point in opposing directions. 
These conflicts appear intermittently and can be severe, causing oscillations, forgetting, and ``seesaw'' trade-offs where improvements in one objective coincide with regressions in others. 
Together, variance hijacking and gradient conflicts explain why static scalarization often yields unstable training and poor Pareto trade-offs in T2I alignment.

These observations suggest that effective multi-objective alignment requires \textit{state-dependent} scheduling, adjusting priorities based on training dynamics rather than fixed weights. 
To this end, we propose \textbf{APEX} (\textbf{A}daptive \textbf{P}riority-based \textbf{E}fficient \textbf{X}-objective Alignment), a simple yet effective framework that addresses both failure modes without discarding much generated samples (unlike sample-filtering approaches such as Parrot \citep{lee2024parrot}). 
APEX decouples two roles that are conflated in static scalarization: (i) constructing a stable scalar learning signal under heterogeneous rewards, and (ii) deciding which objectives to emphasize at each stage of training. 
Specifically, APEX introduces \textbf{Dual-Stage Adaptive Normalization (DSAN)} to neutralize variance hijacking by standardizing rewards per objective and re-normalizing after aggregation, keeping the effective update scale stable even as priorities change. 
On top of this calibrated signal space, the \textbf{$\bm{\mathcal{P}^3}$ mechanism} computes adaptive priorities from learning potential, inter-objective conflicts, and remaining headroom to empirical upper bounds, dynamically steering optimization toward bottlenecks while damping destructive interference. 

In summary, our contributions are as follows: 
\textbf{First}, we provide a mechanistic analysis of why static scalarization fails in multi-objective T2I RL, identifying \textbf{variance hijacking} and \textbf{gradient conflicts} as causes of optimization imbalance.
\textbf{Second}, we propose \textbf{APEX}, a decoupled framework combining DSAN (stable normalization under heterogeneous rewards) with $\bm{\mathcal{P}^3}$ (dynamic priority scheduling from training-state signals).
\textbf{Third}, experiments on Stable Diffusion 3.5 demonstrate improved Pareto trade-offs across OCR, Aesthetic, PickScore, and DeQA, achieving substantially higher hypervolume than static scalarization while maintaining full sample efficiency.

\section{Related Work}

\noindent \textbf{RL for T2I Alignment.} 
As T2I architectures evolve from Latent Diffusion Models (LDMs) \citep{Rombach_2022_CVPR} to Flow Matching frameworks \citep{lipman2023flowmatchinggenerativemodeling}, the research focus has shifted from high-fidelity image synthesis to precise alignment with multi-faceted human intents. 
Reinforcement Learning (RL) has emerged as the core paradigm for post-training alignment: 
methods like DPOK \citep{NEURIPS2023_fc65fab8} and DDPO \citep{black2024trainingdiffusionmodelsreinforcement} stabilize training via KL regularization and policy optimization. 
Human preference benchmarks \citep{NEURIPS2023_33646ef0,kirstain2023pickapicopendatasetuser} provide unified reward signals; DRaFT \citep{clark2024directlyfinetuningdiffusionmodels} validated static linear scalarization for multiple rewards. Notably, Flow-GRPO \citep{liu2025flowgrpotrainingflowmatching} successfully extended Group Relative Policy Optimization to flow matching models, improving single-step efficiency. However, these methods remain primarily single-objective driven. In contrast, our \textbf{APEX} mechanism introduces dynamic priority scheduling, extending optimization to complex heterogeneous multi-objective scenarios.

\vspace{0.5em}
\noindent \textbf{Multi-Objective Optimization.} 
Finding a Pareto-optimal balance between conflicting objectives remains a central challenge. Parrot \citep{lee2024parrot} approximates the Pareto front through non-dominated sorting (NSGA-II \citealt{deb2002fast}), but relies on rejection sampling with significant sample wastage. T2I-R1 \citep{jiang2025t2ir1reinforcingimagegeneration} integrates Chain-of-Thought for semantic planning, yet employs fixed ensemble averaging without dynamic priority allocation. In the LLM domain, \citet{lu2025learningoptimizemultiobjectivealignment} addresses the failure of fixed-weight linear scalarization by introducing dynamic reward weighting. Concurrent work \citep{lyu2025multigrpomultigroupadvantageestimation} explores independent reward normalization. Inspired by these, APEX addresses unique T2I challenges: cross-modal heterogeneous rewards (e.g., discrete OCR versus smooth aesthetics) cause severe ``variance hijacking,'' where high-variance signals implicitly dominate optimization. APEX mitigates this through \textbf{DSAN} for signal calibration and a $\bm{\mathcal{P}^3}$ priority scheduler that fuses gradient geometry with \textbf{Utopia Point} \citep{Marler2004} metrics to guide toward the Pareto front efficiently.

\begin{figure*}[t]
  \centering
  \includegraphics[width=\textwidth]{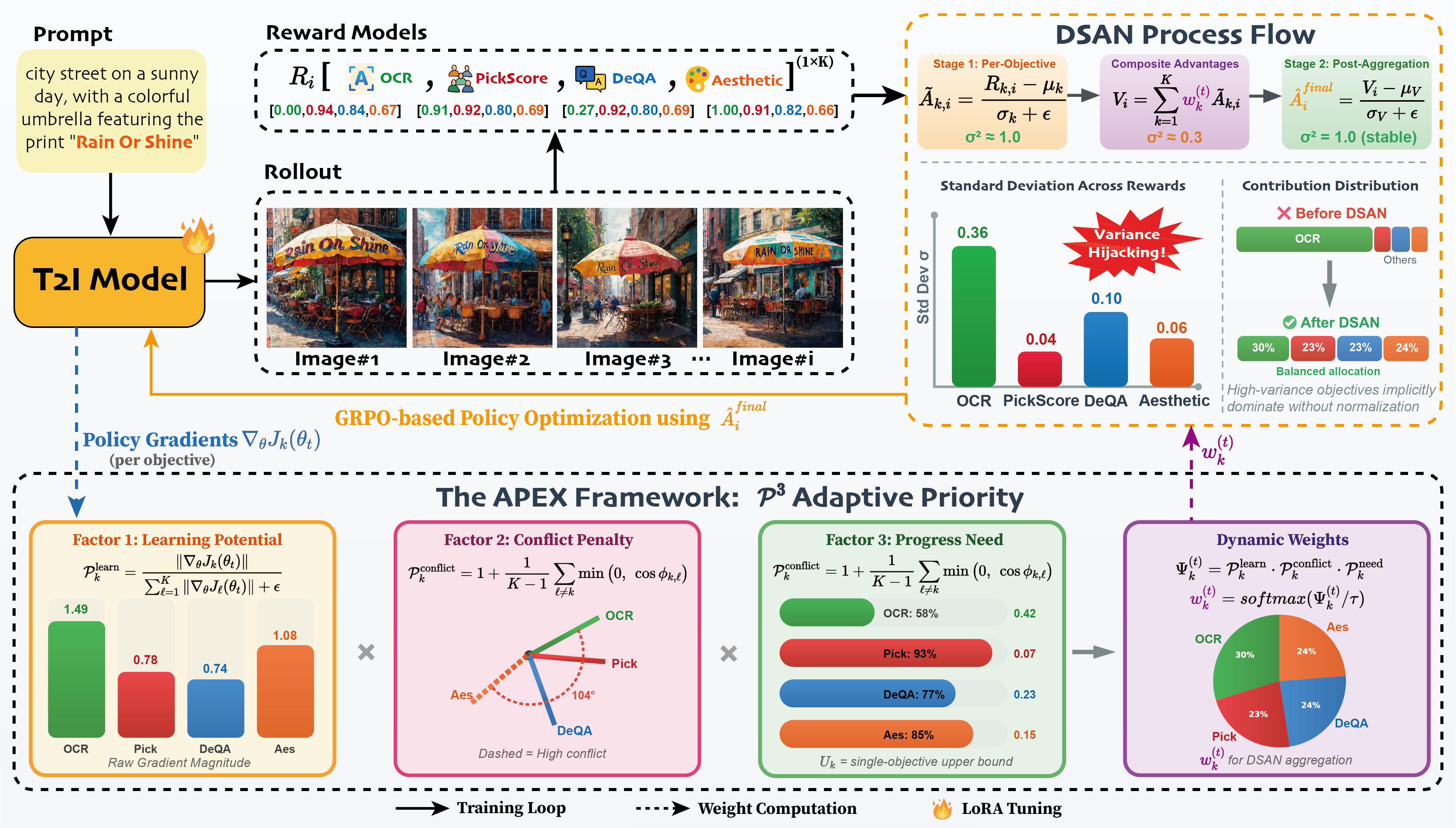}
  \caption{\textbf{Overview of the APEX framework.} \textit{Top:} The training loop generates rollouts from prompts, evaluates them with multiple reward models, and performs GRPO-based policy optimization. DSAN eliminates variance hijacking via dual-stage normalization, producing balanced gradient contributions. \textit{Bottom:} The $\mathcal{P}^3$ mechanism analyzes per-objective policy gradients $\nabla_\theta J_k(\theta_t)$ (used only for weight computation, not 
parameter updates) to compute dynamic weights $w_k^{(t)}$ by fusing learning potential, conflict penalty, and progress need, which are subsequently fed back to DSAN for advantage aggregation.}
  \label{fig:main}
  \vspace{-3mm} 
\end{figure*}

\section{Method}

\subsection{Preliminaries}
\label{subsec:preliminaries}

We consider a flow-matching-based text-to-image model parameterized by $\theta$. 
Given a prompt $c \sim \mathcal{D}$, the model generates an image $x_0 \in \mathbb{R}^{H\times W\times 3}$ by denoising from an initial latent $x_1$ along a continuous time variable $j\in[0,1]$ (with $j=1$ being noise and $j=0$ being data). 
In practice, we discretize time into $T$ steps $\{j_t\}_{t=0}^{T}$ with $j_0=0$ and $j_T=1$, and denote the discrete trajectory as $\tau = \{x_{j_T}, x_{j_{T-1}}, \dots, x_{j_0}\}$.

Standard flow matching inference follows a deterministic reverse-time ODE: $d x_j = v_\theta(x_j, j)\, dj$, which is unsuitable for policy-gradient-style alignment due to the lack of stochastic exploration and an explicit tractable transition density.
To enable stochastic sampling while preserving the model's marginal distribution, 
we adopt the ODE-to-SDE conversion from recent work on stochastic flow matching \citep{albergo2023stochastic, liu2025flowgrpotrainingflowmatching}:
\begin{equation}
\label{eq:fm_sde}
\begin{split}
d x_j &= v_\theta(x_j, j)\, dj \\
&\quad + \frac{\sigma_j^2}{2j} \left(x_j + (1-j)\, v_\theta(x_j, j)\right) dj
+ \sigma_j\, d w,
\end{split}
\end{equation}
where $w$ is a standard Wiener process and $\sigma_j$ controls stochasticity. 
This SDE is constructed so that its marginal distribution matches the original ODE model under suitable conditions.
Details and derivations are provided in Appendix~\ref{app:sde_derivation}. 

We discretize Eq.~\eqref{eq:fm_sde} using a numerical SDE solver (e.g., Euler--Maruyama), yielding a stochastic Markov chain with Gaussian transitions:
\begin{equation}
\label{eq:gauss_transition}
\pi_\theta(x_{j_{t-1}} \mid x_{j_t}, c) = \mathcal{N}\!\left(x_{j_{t-1}};\ \mu_\theta(x_{j_t}, j_t, c),\ \Sigma_{j_t}\right).
\end{equation}
The closed-form Gaussian density enables tractable computation of (i) per-step log-probabilities $\log \pi_\theta(x_{j_{t-1}}\mid x_{j_t}, c)$, (ii) likelihood ratios used by PPO/GRPO, and (iii) analytical KL divergence to a reference policy $\pi_{\text{ref}}$ (Appendix~\ref{app:kl_implementation}).

\paragraph{GRPO for flow matching policies.}
To align $\pi_\theta$ without training an additional critic, we adopt 
Group Relative Policy Optimization (GRPO)~\citep{shao2024deepseekmathpushinglimitsmathematical}. 
For each prompt $c$, we sample a group of $G$ independent trajectories 
under the stochastic policy induced by Eq.~\eqref{eq:gauss_transition}, 
obtaining images $\{x_0^{(i)}\}_{i=1}^G$.
Let $\theta_{\text{old}}$ be the behavior policy. The per-step likelihood 
ratio is

\begin{equation}
\label{eq:ratio}
r_t^{(i)}(\theta)=
\frac{\pi_\theta(x_{j_{t-1}}^{(i)}\mid x_{j_t}^{(i)},c)}
{\pi_{\theta_{\text{old}}}(x_{j_{t-1}}^{(i)}\mid x_{j_t}^{(i)},c)}.
\end{equation}
GRPO maximizes a clipped objective, 
regularized by KL divergence to the reference policy 
(see Appendix~\ref{app:grpo_details} for the complete formulation).

\subsection{Problem Formulation}
\label{sec:problem}

Given a prompt $c\sim\mathcal D$, the policy $\pi_\theta$ samples an image $x_0\sim \pi_\theta(\cdot\mid c)$. 
We are provided with $K$ heterogeneous reward functions $\{R_k(x_0,c)\}_{k=1}^K$ (e.g., text-image faithfulness, aesthetics, OCR quality), and we aim to improve all of them during fine-tuning.
Define the expected performance of each objective
\begin{equation}
\label{eq:Jk_def}
J_k(\theta)\triangleq \mathbb{E}_{c\sim\mathcal D,\ x_0\sim\pi_\theta(\cdot\mid c)}\!\left[R_k(x_0,c)\right],
\end{equation}
where $k\in\{1,\dots,K\}$.
The alignment goal is naturally multi-objective, i.e., to improve the vector $\big[J_1(\theta),\dots,J_K(\theta)\big]^\top$. 
Since our optimizer operates on a scalar loss, a common practice is to optimize a scalarization of objectives:
\begin{equation}
\label{eq:scalar_J_def}
J(\theta;\mathbf w)\triangleq \sum_{k=1}^K w_k\, J_k(\theta),
\end{equation}
where $\mathbf w\in\Delta^{K-1}=\{\mathbf w:\ w_k\ge 0,\ \sum_{k=1}^K w_k=1\}$.
Equivalently, at the sample level this corresponds to a scalarized reward
\begin{equation}
\label{eq:scalar_reward_def}
R(x_0,c;\mathbf w)=\sum_{k=1}^K w_k\, R_k(x_0,c).
\end{equation}
All rewards are defined on the final image $x_0$ (terminal reward), while the policy is the stochastic reverse-time trajectory induced by Eq.~\eqref{eq:gauss_transition}.

\paragraph{Why dynamic weights are needed?}
Most prior pipelines use a fixed $\mathbf w$ throughout training, such as DRaFT \citep{clark2024directlyfinetuningdiffusionmodels}. 
In practice, a fixed $\mathbf w$ does not guarantee balanced progress across objectives because the \emph{optimization state} changes over time.
Concretely, (i) some objectives may currently provide stronger learning signals than others, (ii) objectives can interact through shared parameters and may help or hinder each other, and (iii) some objectives may plateau earlier and benefit less from continued emphasis.
These effects evolve during fine-tuning, motivating a state-dependent weighting rule $\mathbf w^{(t)}$ that adapts to the current training dynamics, where $t$ denotes the training step index.

Under GRPO, a direct instantiation plugs the scalarized reward in Eq.~\eqref{eq:scalar_reward_def} into group-relative advantage normalization:
\begin{equation}
\label{eq:naive_mo_grpo}
\begin{split}
    R^{(i)} &= \sum_{k=1}^K w_k^{(t)}\,R_k(x_0^{(i)},c), \\
    \hat A^{(i)} &= \frac{R^{(i)}-\text{mean}(\{R^{(m)}\}_{m=1}^G)}
    {\text{std}(\{R^{(m)}\}_{m=1}^G)+\epsilon}.
\end{split}
\end{equation}
While simple, this \textit{weight-then-normalize} baseline can be unstable under heterogeneous rewards and time-varying weights: objectives with larger dispersion can disproportionately shape the normalized advantage, and changing $\mathbf w^{(t)}$ can shift the advantage statistics and effectively rescale the update from one iteration to the next.
This motivates (i) an advantage construction that is robust to reward heterogeneity, and (ii) a principled mechanism for updating $\mathbf w^{(t)}$ from the observed optimization state.

\subsection{The APEX Method}
\label{sec:apex}

We propose \textbf{APEX}, a dynamic multi-objective alignment method built on GRPO.
APEX is designed as a two-level solution that jointly stabilizes the \textit{training signal} and adapts the \textit{objective priorities}:
\textit{(i) Dual-Stage Adaptive Normalization (DSAN)} constructs a scalar advantage whose scale is stable under heterogeneous rewards and changing weights, and
\textit{(ii) the $\bm{\mathcal{P}^3}$ mechanism} updates weights from the current learning signal strength, inter-objective interaction, and remaining improvement room.
Together, DSAN and $\bm{\mathcal{P}^3}$ form a closed loop: $\bm{\mathcal{P}^3}$ adjusts what to emphasize, while DSAN ensures the resulting scalar advantage remains comparable across iterations and does not inadvertently change the effective update strength. 
An overview of APEX is provided in Figure \ref{fig:main}.

\subsubsection{Dual-Stage Adaptive Normalization}
\label{sec:dsan}

\paragraph{Stage 1. Per-objective group standardization.}
For each objective $k$, we compute a group-relative standardized advantage:
\begin{equation}
\label{eq:dsan_stage1}
\tilde{A}_{k}^{(i)}=
\frac{R_k(x_0^{(i)},c)-\text{mean}(\{R_k(x_0^{(m)},c)\}_{m=1}^G)}
{\text{std}(\{R_k(x_0^{(m)},c)\}_{m=1}^G)+\epsilon},
\end{equation}
where $\epsilon$ is a small constant for numerical stability.
This aligns objectives onto a comparable scale so that no single reward dominates the update merely due to scale or dispersion.

\paragraph{Stage 2. Post-aggregation normalization.}
Given weights $\mathbf w^{(t)}$, we aggregate standardized advantages and normalize again:
\begin{equation}
\label{eq:dsan_stage2}
\begin{split}
    V^{(i)} &= \sum_{k=1}^K w_k^{(t)}\,\tilde{A}_{k}^{(i)}, \\
    \hat{A}_{\text{final}}^{(i)} &= \frac{V^{(i)}-\text{mean}(\{V^{(m)}\}_{m=1}^G)}
    {\text{std}(\{V^{(m)}\}_{m=1}^G)+\epsilon}.
\end{split}
\end{equation}
Stage~2 makes the overall advantage distribution stable even when $\mathbf w^{(t)}$ changes, preventing inadvertent iteration-to-iteration rescaling of the effective GRPO update. Empirical evidence of advantage 
variance dynamics is provided in Appendix~\ref{app:dynamics}.

\subsubsection{\texorpdfstring{$\bm{\mathcal{P}^3}$}{P3} Adaptive Priority Mechanism}
\label{sec:p3}

APEX assigns each objective $k$ a priority score $\Psi_k^{(t)}$ and maps priorities to weights through softmax:
\begin{equation}
\label{eq:p3_priority}
\begin{split}
    \Psi_k^{(t)} &= \mathcal{P}^{\text{learn}}_k \cdot \mathcal{P}^{\text{conflict}}_k \cdot \mathcal{P}^{\text{need}}_k, \\
    w_k^{(t)} &= \frac{\exp(\Psi_k^{(t)}/\tau)}{\sum_{\ell=1}^K \exp(\Psi_\ell^{(t)}/\tau)}.
\end{split}
\end{equation}

\bm{$\mathcal{P}^3$} is motivated from the local dynamics of
optimizing the scalarized objective $J(\theta; \mathbf{w}) =$
$\sum_k w_k J_k(\theta)$. The multiplicative aggregation follows 
non-compensatory selection principles in multi-criteria 
optimization~\citep{Marler2004} (derivation in Appendix~\ref{app:multiplication}).
A gradient step gives $\Delta\theta \propto$
$\sum_\ell w_\ell\nabla J_\ell$, and the first-order change of objective
$k$ is
\begin{equation}
\label{eq:deltaJk_in_p3}
\begin{split}
    \Delta J_k &\approx \nabla J_k^\top \Delta\theta \\
    &= \eta \sum_{\ell=1}^K w_\ell\ \nabla J_k^\top \nabla J_\ell \\
    &= \eta \bigg( w_k \|\nabla J_k\|^2 \\
    &\quad + \sum_{\ell\neq k} w_\ell \|\nabla J_k\|\,\|\nabla J_\ell\|\cos\phi_{k,\ell} \bigg).
\end{split}
\end{equation}
Eq.~\eqref{eq:deltaJk_in_p3} suggests using (i) gradient magnitude as a \textit{learning-signal strength} proxy and (ii) cosine similarity as an \textit{inter-objective interaction} (synergy/conflict) proxy. Since this local approximation does not reflect \textit{long-term objective saturation}, we further introduce a \textit{progress need} signal based on distance to an empirical upper bound.

\paragraph{Learning potential (LP).}
We define
\begin{equation}
\label{eq:p3_learn}
\mathcal{P}^{\text{learn}}_k=
\frac{\|\nabla_\theta J_k(\theta_t)\|}
{\sum_{\ell=1}^K \|\nabla_\theta J_\ell(\theta_t)\|+\epsilon},
\end{equation}
where $\epsilon$ is a small constant (set to $10^{-8}$) for 
numerical stability, preventing division by zero.
Gradient estimation details are provided in Appendix~\ref{app:implementation}.

\paragraph{Conflict penalty (CP).}
We penalize objectives that conflict with others:

\begin{equation}
\label{eq:p3_conflict}
\begin{split}
    \mathcal{P}^{\text{conflict}}_k &= 1 + \frac{1}{K-1}\sum_{\ell\neq k}\min\big(0,\ \cos\phi_{k,\ell}\big), \\[2pt]
    \cos\phi_{k,\ell} &= \frac{\langle \nabla J_k,\nabla J_\ell\rangle}
    {\|\nabla J_k\|\,\|\nabla J_\ell\|+\epsilon}.
\end{split}\raisetag{20pt}
\end{equation}


\paragraph{Progress need (PN).}
Let $U_k$ be an empirical upper bound (utopia point) and $\bar{R}_k^{(t)}$ be a running performance estimate. We define

\begin{equation}
\label{eq:p3_need}
\mathcal{P}^{\text{need}}_k=
1 + \max\Big(0,\ \frac{U_k-\bar{R}_k^{(t)}}{U_k+\epsilon}\Big).
\end{equation}
Estimation details for $U_k$ and $\bar{R}_k^{(t)}$ are provided in 
Appendix~\ref{app:implementation}, enabling bottleneck identification.

\paragraph{Default Settings.}
The ${\mathcal{P}^3}$ mechanism is designed for adaptive scheduling rather than 
monotonic convergence. It possesses formal stability guarantees: 
softmax normalization ensures weights never vanish or concentrate 
on a single objective, with weight ratios bounded by $\exp(2/\tau)$ 
(Appendix~\ref{app:theory} for proofs).

Unless otherwise stated, we set $\tau=1$.
Putting everything together, APEX replaces the naive advantage in Eq.~\eqref{eq:naive_mo_grpo} with DSAN (Eqs.~\eqref{eq:dsan_stage1}--\eqref{eq:dsan_stage2}) and updates weights using $\bm{\mathcal{P}^3}$ (Eq.~\eqref{eq:p3_priority}).

\begin{table*}[t]
\centering
\small
\begin{tabular}{lccccc}
\toprule
\multirow{2}{*}{\textbf{Model}} & \textbf{Text Rendering} & \textbf{Human Pref.} & \multicolumn{2}{c}{\textbf{Image Quality}} & \multirow{2}{*}{\textbf{Hypervolume}} \\
\cmidrule(lr){2-2} \cmidrule(lr){3-3} \cmidrule(lr){4-5}
& OCR Acc. ($\uparrow$) & PickScore ($\uparrow$) & DeQA ($\uparrow$) & Aesthetic ($\uparrow$) & ($\uparrow$) \\
\midrule
\multicolumn{6}{c}{\textit{Base Model}} \\
\midrule
SD3.5-M & 0.59 & 21.72 & 4.07 & 5.39 & - \\
\midrule
\multicolumn{6}{c}{\textit{Single-Objective Specialists}} \\
\midrule
Flow-GRPO (OCR-Only) & \textbf{0.92} & 22.44 & 4.06 & 5.32 & 0.00 \\
Flow-GRPO (PickScore-Only) & 0.69 & \textbf{23.53} & 4.22 & \textbf{5.92} & 1.11 \\
\midrule
\multicolumn{6}{c}{\textit{Multi-Objective Generalists}} \\
\midrule
Flow-GRPO (Static-Weight) & \underline{0.88} & 22.62 & 4.24 & 5.51 & 0.41 \\
\textbf{APEX (Ours)} & 0.83 & \underline{23.03} & \textbf{4.42} & \textbf{5.92} & \textbf{4.49} \\
\bottomrule
\end{tabular}
\caption{\textbf{Main Results.} Performance comparison on text rendering (evaluated on OCR test set), as well as human preference and image quality (both evaluated on DrawBench). Hypervolume approximates each model's Pareto contribution as the product of normalized improvements over the base model (SD3.5-M). Metrics are normalized for HV computation: OCR$\in[0,1]$ (inherently scaled), 
PickScore/26, DeQA/5, Aesthetic/10 (see Appendix~\ref{app:hypervolume}).
\textbf{Bold} indicates the best performance across all models (ties included), while \underline{underline} marks the best within the multi-objective group when it differs from the global best.}
\label{tab:main_results}
\end{table*}

\section{Experiment}

We evaluate APEX on Stable Diffusion 3.5 Medium (SD3.5-m)~\citep{esser2024scaling} across four heterogeneous objectives to answer: 
(\textbf{i}) Does APEX improve Pareto trade-offs over static weighting? 
(\textbf{ii}) How does APEX resolve variance hijacking and gradient conflicts? 
(\textbf{iii}) Are DSAN and $\mathcal{P}^3$ both necessary?
Section~\ref{sec:setup} describes experimental setup; 
Section~\ref{sec:main} presents main results; 
Section~\ref{sec:analysis} analyzes training dynamics; 
Section~\ref{sec:ablations} validates each component via ablation studies.

\subsection{Experimental Setup}
\label{sec:setup}

\paragraph{Objectives.}
Four heterogeneous reward functions are selected, spanning from structural constraints to perceptual qualities: 
(1) \textbf{OCR}~\citep{NEURIPS2022_ec795aea}: 
measuring text fidelity via normalized Levenshtein distance, representing discrete structural constraints with high variance; 
(2) \textbf{PickScore}~\citep{kirstain2023pickapicopendatasetuser}: human preference model for image-text alignment trained on large-scale feedback; 
(3) \textbf{DeQA}~\citep{You_2025_CVPR}: multi-modal LLM-based metric quantifying distortions and low-level artifacts; 
(4) \textbf{Aesthetic Score}~\citep{NEURIPS2022_a1859deb}: CLIP-based regressor for aesthetic appeal.
OCR is evaluated on the Flow-GRPO test set~\citep{liu2025flowgrpotrainingflowmatching}, while others use DrawBench~\citep{NEURIPS2022_ec795aea}.

\paragraph{Baselines.}
We compare APEX against Static Linear Scalarization, the standard approach using equal fixed weights ($w_i{=}1/K$), a common baseline isolating the effect of dynamic weighting. 
Single-Objective Specialists (optimized for individual rewards) are reported as performance bounds.

\paragraph{Implementation Details.}
We utilize LoRA~\citep{hu2021loralowrankadaptationlarge} for efficient fine-tuning with GRPO~\citep{shao2024deepseekmathpushinglimitsmathematical} adapted for flow matching to enable fast training on 8x NVIDIA A100 GPUs. 
Training employs 10 denoising steps (vs. 40 for inference) for efficiency. 
Key hyperparamters are: group size $G{=}24$, learning rate $3{\times}10^{-4}$, temperature $\tau{=}1$. 
We report single-run results following standard practice~\citep{black2024trainingdiffusionmodelsreinforcement,clark2024directlyfinetuningdiffusionmodels}.
See Appendix~\ref{app:implementation} for full details.

\subsection{Main Results}
\label{sec:main}

\subsubsection{Quantitative Evaluation}
Table~\ref{tab:main_results} compares APEX against baselines across four objectives. 
We report the un-tuned SD3.5-M (Base) model and Single-Objective Specialists to establish performance bounds. 

\vspace{0.5em}
\noindent \textbf{Limitations of Single-Objective Optimization.} 
Specialists achieve peak performance in their target domains at the cost 
of other objectives. 
The OCR-Only model's aesthetic score (5.32) and DeQA (4.06) both fall 
below the base model; such regressions yield zero hypervolume 
despite substantial OCR gains.
This confirms that single-objective RL causes excessive optimization 
of discrete structural constraints while sacrificing perceptual quality.

\vspace{0.5em}
\noindent \textbf{Variance Hijacking in Static Weighting.} 
The static baseline exhibits optimization imbalance. Although reaching OCR accuracy of 0.88, its gains in PickScore (+0.90 over base) and DeQA (+0.17) lag behind APEX (+1.31 and +0.35 respectively), while its aesthetic score (5.51) shows minimal improvement. 
This supports: without scale calibration, high-variance OCR signals implicitly ``hijack'' the optimization trajectory, starving low-variance objectives of gradient budget.

\vspace{0.5em}
\noindent \textbf{Pareto Advancement via APEX.} 
APEX achieves superior multi-objective balance through DSAN and $\mathcal{P}^3$. 
While maintaining competitive OCR (0.83 vs. 0.88 for static weighting), 
APEX substantially improves perceptual objectives: PickScore reaches 23.03 
(98\% of the specialist), DeQA achieves 4.42 (highest among all models), 
and Aesthetic attains 5.92 (tied for best). 
This demonstrates that adaptive scheduling enables near-specialist performance 
on individual metrics without sacrificing multi-objective balance.
APEX's hypervolume ($4.49{\times}10^{-5}$, 10.9$\times$ static weighting) 
quantifies this Pareto dominance (Appendix~\ref{app:hypervolume}).

\vspace{0.5em}
\subsubsection{Qualitative Validation}
Visual comparisons (Figure~\ref{fig:comparison}, Appendix~\ref{app:qualitative}) 
validate the quantitative findings. 
APEX demonstrates improved coordination of text rendering, semantic coherence, 
and visual context, addressing failure modes such as spelling errors and 
style inconsistencies that occasionally appear in Static-Weight generations.
In perceptual quality, APEX shows enhanced lighting modeling, color interaction, 
and physical plausibility, while Static-Weight exhibits more limited gains 
over the base model in these fine-grained aesthetic dimensions.

\subsection{Analysis of Training Dynamics}
\label{sec:analysis}

We analyze training dynamics to reveal why static scalarization fails 
and how APEX achieves adaptive multi-objective scheduling.

\begin{figure}[t]
  \centering
  \includegraphics[width=\columnwidth]{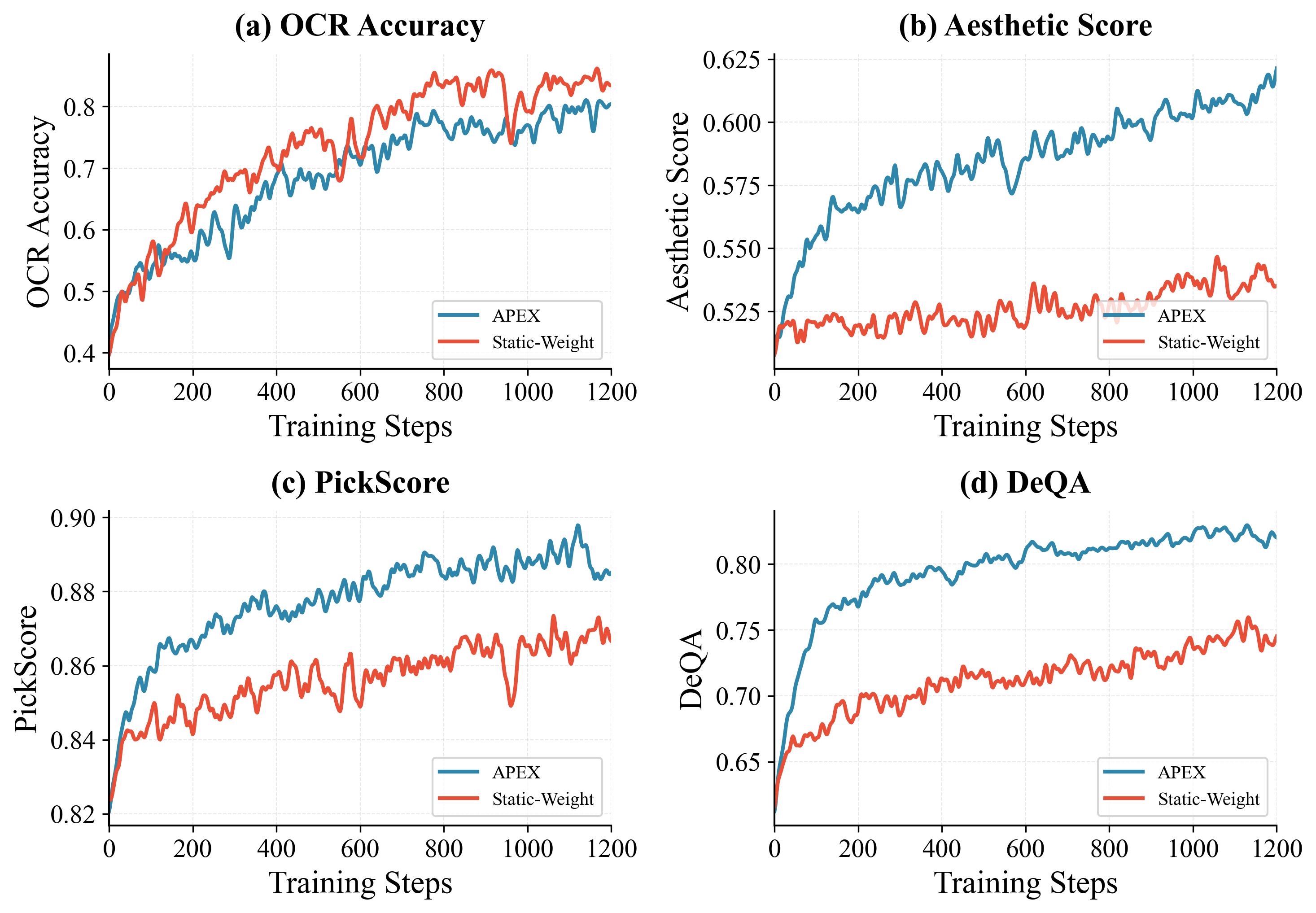} 
  \caption{\textbf{Training dynamics revealing variance hijacking.} Four subplots track four reward objectives across training steps, comparing APEX (blue) and Static-Weight baseline (red). The baseline shows OCR plateauing while Aesthetic stagnates, whereas APEX achieves balanced growth across all dimensions.}
  \label{fig:pathology}
\end{figure}

\vspace{0.5em}
\noindent \textbf{Revealing the Pathology of Variance Hijacking.} 

By comparing the training trajectories of APEX and the Static-Weight 
baseline (Fig.~\ref{fig:pathology}), we demonstrate the variance hijacking 
phenomenon.
In the static baseline, OCR accuracy plateaus near $\sim$0.88 
while normalized Aesthetic Score stagnates at $\sim$0.54 throughout training.  
Even after OCR plateaus, other objectives (image quality, human 
preference) fail to improve, indicating that high-variance OCR signals 
continue to dominate gradient updates, leaving low-variance objectives 
under-optimized.
In contrast, APEX achieves synchronized growth 
across all dimensions through DSAN's dual-stage normalization; Stage 2 
renormalization is verified by variance decay analysis in 
Appendix~\ref{app:dynamics}.

\begin{figure}[t]
  \centering
  \includegraphics[width=\columnwidth]{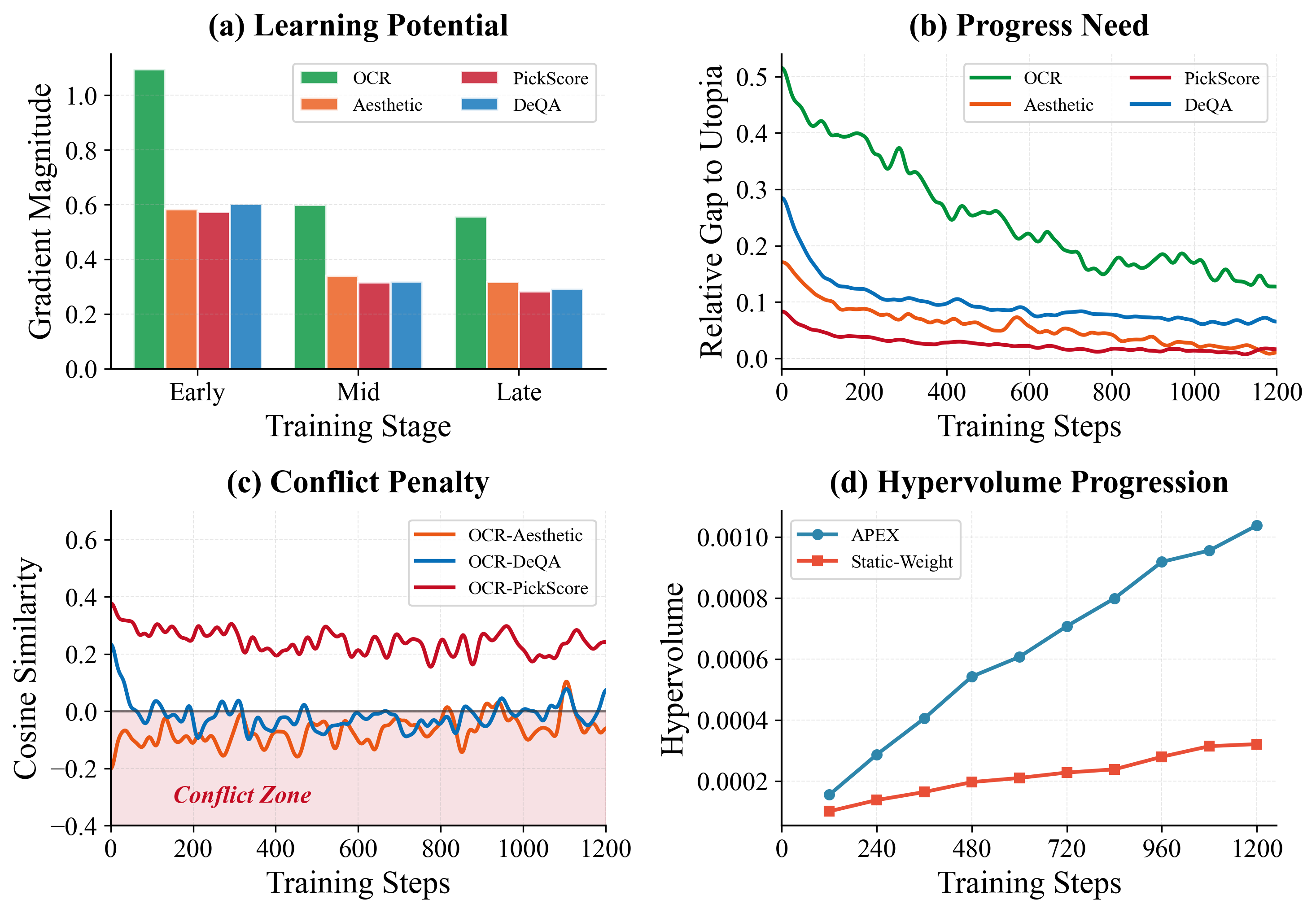}
  \caption{\textbf{Analysis of $\bm{\mathcal{P}^3}$ dynamics and Hypervolume progression.} (a–c) The three $\mathcal{P}^3$ factors—Learning Potential, Progress Need, and Conflict Penalty—which jointly guide adaptive weight scheduling. (d) Cumulative Hypervolume comparison showing APEX achieves 3.2$\times$ the dominated space volume of the static baseline.}
  \label{fig:analysis}
\end{figure}

\vspace{0.5em}
\noindent \textbf{Observations on the $\bm{\mathcal{P}^3}$ Mechanism.} 
We deconstruct how the three $\mathcal{P}^3$ factors jointly guide 
adaptive scheduling (Figure~\ref{fig:analysis}a-c). The \textbf{LP factor} 
tracks gradient magnitudes: OCR consistently exhibits high gradient norms, 
reflecting strong parameter sensitivity, and LP accordingly assigns higher 
base priority. The \textbf{PN factor} complements this by monitoring 
distance to empirical upper bounds (\textbf{Utopia Points}); objectives 
far from saturation receive additional emphasis, redirecting optimization 
toward bottleneck dimensions with maximal improvement potential. The 
\textbf{CP factor} addresses a distinct challenge: gradient conflicts 
exhibit ``bursty'' behavior, with intermittent severe negative correlations 
between objectives. CP acts as a dynamic shock absorber, detecting these 
conflict episodes and temporarily downweighting interfering objectives to 
prevent destructive gradient interference. 
Together, these three factors allow APEX to balance exploitation (LP), 
exploration of underperforming dimensions (PN), and conflict avoidance (CP).

\vspace{0.5em}
\noindent \textbf{Overall Capability Boundary Analysis.} 
Finally, we quantify overall optimization effectiveness using the Hypervolume 
(HV) metric (Figure~\ref{fig:analysis}d), which measures the volume of 
objective space dominated by a model's Pareto set. We evaluate 10 evenly-spaced 
checkpoints on a held-out test set, using early-training performance as the 
reference point (see Appendix~\ref{app:hypervolume} for details). The static 
baseline's hypervolume growth decelerates significantly in later stages 
(final HV $\approx 0.0003$), confirming its inability to escape the performance 
ceiling imposed by variance hijacking. In contrast, APEX maintains robust 
growth throughout training, ultimately achieving 3.2$\times$ the baseline's 
hypervolume (HV $\approx 0.0010$). 
This 3.2$\times$ improvement demonstrates that APEX successfully expands 
the Pareto frontier, improving perceptual objectives without sacrificing 
OCR capability.

\subsection{Ablation Studies}
\label{sec:ablations}

\begin{figure}[t]
  \centering
  \includegraphics[width=\columnwidth]{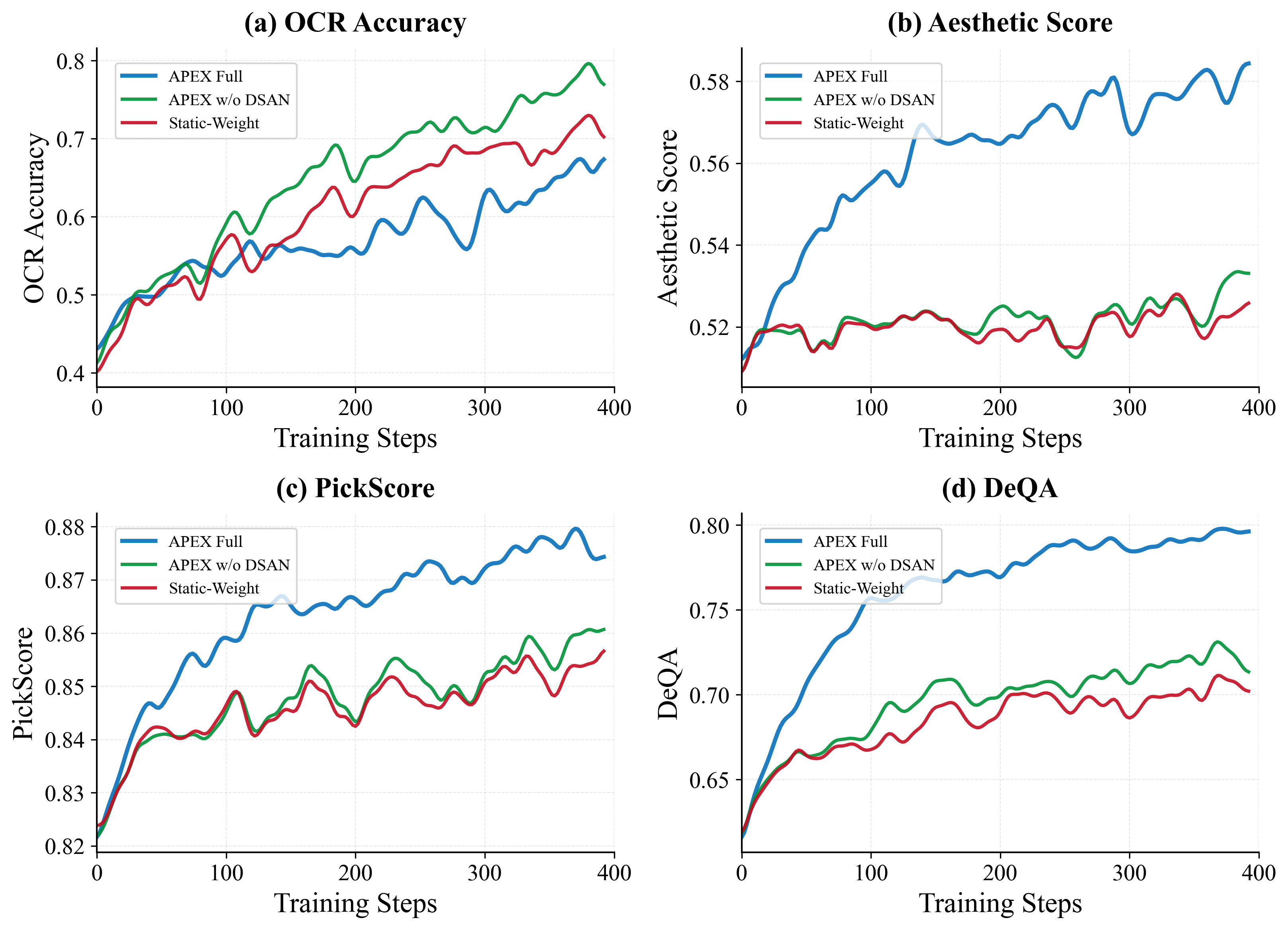}
  \caption{\textbf{Ablation study on DSAN.} Reward trajectories over the first 400 steps, comparing APEX Full, APEX w/o DSAN, and Static-Weight. Removing DSAN degrades convergence due to variance hijacking, yet the $\mathcal{P}^3$ mechanism still outperforms static weighting. This confirms that DSAN is essential for the adaptive priority mechanism to reach peak performance.}
  \label{fig:dsan_ablation}
\end{figure}

To verify the necessity of each component in APEX, we conduct two sets 
of ablation experiments: (1) removing DSAN to validate the scale calibration 
module, and (2) individually ablating each $\mathcal{P}^3$ factor to quantify 
their contributions. 
Due to computational constraints, we focus on early-to-mid training dynamics 
where component effects are most pronounced; extended stability analysis is 
provided in Appendix~\ref{app:ablation}.

\vspace{0.5em}
\noindent \textbf{Necessity of Scale Calibration (DSAN).} 
We compare APEX with a variant removing DSAN (\textit{APEX w/o DSAN}). 
Figure~\ref{fig:dsan_ablation} shows that both w/o DSAN and Static-Weight 
suffer from variance hijacking: perceptual objectives (Aesthetic, PickScore, 
DeQA in subplots b-d) stagnate throughout training.
However, w/o DSAN consistently outperforms Static-Weight across all objectives 
(Figure~\ref{fig:dsan_ablation}), with the most pronounced improvement in OCR 
(subplot a). This suggests that $\mathcal{P}^3$ improves weight allocation 
even without scale calibration.
Yet APEX Full substantially outperforms w/o DSAN 
across all perceptual dimensions.
This confirms that while $\mathcal{P}^3$ provides incremental benefits 
independently, DSAN is essential to fully overcome variance hijacking 
and unlock the full potential of adaptive multi-objective scheduling.

\begin{table}
\centering
\small
\begin{tabular}{lcc}
\toprule
\textbf{Variant} & \textbf{Cumulative HV} & \textbf{$\Delta$ vs Full} \\
\midrule
APEX  & $4.26 \times 10^{-4}$ & — \\
\quad \textit{w/o LP} & $3.73 \times 10^{-4}$ & $-12.5\%$ \\
\quad \textit{w/o CP} & $3.65 \times 10^{-4}$ & $-14.5\%$ \\
\quad \textit{w/o PN} & $3.82 \times 10^{-4}$ & $-10.3\%$ \\
\bottomrule
\end{tabular}
\caption{\textbf{$\bm{\mathcal{P}^3}$ Factor Ablation.} Cumulative Hypervolume 
computed over training rewards (50-step averaging window) up to step 600 
(see Appendix~\ref{app:hypervolume} for details). 
Removing any factor leads to 10--15\% degradation, with conflict penalty 
(CP) showing the largest impact.}

\label{tab:p3_ablation}
\end{table}

\vspace{0.5em}
\noindent \textbf{Contribution of $\bm{\mathcal{P}^3}$ Factors.} 
Table~\ref{tab:p3_ablation} presents cumulative Hypervolume when individually 
removing each factor. 
\textbf{Conflict penalty (CP)} has the largest impact: removing it causes 
$-14.5\%$ degradation, as competing gradient directions trigger oscillations 
without CP's damping effect.
\textbf{Learning potential (LP)} ablation leads to $-12.5\%$ drop, 
demonstrating the importance of gradient-based priority scheduling. 
\textbf{Progress need (PN)} removal results in $-10.3\%$ degradation, 
showing that monitoring distance to upper bounds prevents premature saturation.
The comparable magnitudes (10--15\%) indicate all three factors 
are necessary, addressing a distinct aspect of multi-objective 
optimization.

\section{Conclusion}

In this paper, we introduce APEX to address optimization imbalance in multi-objective 
vision-language alignment, where heterogeneous rewards cause overfit 
to high-variance objectives. Our analysis identified two root causes: 
\textbf{variance hijacking}, where high-variance objectives dominate 
gradient updates, and \textbf{gradient conflicts} between competing directions.
APEX resolves these through Dual-Stage Adaptive Normalization (DSAN) 
and $\mathcal{P}^3$ Adaptive Priorities. On Stable Diffusion 3.5, 
APEX achieves 10.9$\times$ the hypervolume of static scalarization 
with balanced improvements (+1.31 PickScore, +0.35 DeQA, +0.53 Aesthetics) 
while maintaining competitive OCR, providing a principled framework 
for multi-objective alignment.

\clearpage
\section*{Limitations}

Despite its effectiveness, APEX has several limitations that warrant 
future investigation. 
APEX requires gradient estimation for priority computation. While manageable, this overhead scales with the number of objectives.
Moreover, constrained by computational resources, our validation is limited to SD3.5-Medium. 
Our future work will explore generalizing APEX across diverse architectures (SDXL, Flux) and modalities (video, audio), as well as to LLM alignment tasks, which face similar multi-objective optimization challenges.

\bibliography{anthology,custom}

\clearpage
\appendix
\section{Stochastic Flow Matching and GRPO Formulation}
\label{app:sde_derivation}

For completeness, we provide the derivation of the SDE formulation in Section~\ref{subsec:preliminaries}, following \citet{albergo2023stochastic, liu2025flowgrpotrainingflowmatching}. We adapt the notation to our multi-objective framework.

\subsection{From ODE to SDE}

Standard flow matching models use a deterministic ODE for generation:
$dx_j = v_\theta(x_j, j) dj$, where $x_j = (1-j)x_0 + jx_1$ with $x_0 \sim \mathcal{X}_0$ (data) and $x_1 \sim \mathcal{N}(0,I)$ (noise).
To enable stochastic sampling, we construct an SDE with matching marginal distributions.

Consider a forward SDE:
\begin{equation}
\begin{split}
dx_j = \big[ v_\theta(x_j, j) + \tfrac{\sigma_j^2}{2} \nabla \log p_j(x_j) \big] dj \\
+ \sigma_j dw,
\end{split}
\end{equation}
where $\sigma_j$ controls stochasticity. By the Fokker-Planck equation \citep{oksendal2003stochastic}, this SDE preserves the ODE's marginal density $p_j(x)$.

\paragraph{Reverse-time SDE.}
Applying the standard time-reversal formula \citep{anderson1982reverse}:
\begin{equation}
\label{eq:app_reverse}
\begin{split}
dx_j = \big[ v_\theta(x_j, j) - \tfrac{\sigma_j^2}{2} \nabla \log p_j(x_j) \big] dj \\
+ \sigma_j dw.
\end{split}
\end{equation}

\paragraph{Score function for rectified flow.}
For the linear interpolation $x_j = (1-j)x_0 + jx_1$, the conditional score is $\nabla \log p_{j|0}(x_j|x_0) = -x_1/j$. The marginal score is:
\begin{equation}
\label{eq:app_score}
\nabla \log p_j(x_j) = -\tfrac{1}{j}\mathbb{E}[x_1 | x_j].
\end{equation}

From the velocity field definition $v_\theta(x_j,j) = \mathbb{E}[x_1 - x_0 | x_j]$ and the interpolation relation, we derive:
\begin{equation}
\label{eq:app_x1_expect}
\mathbb{E}[x_1 | x_j] = (1-j)v_\theta(x_j,j) + x_j.
\end{equation}

Substituting Eq.~\eqref{eq:app_x1_expect} into Eq.~\eqref{eq:app_score}:
\begin{equation}
\begin{split}
\nabla \log p_j(x_j) = -\tfrac{x_j}{j} - \tfrac{1-j}{j}v_\theta(x_j,j).
\end{split}
\end{equation}

\paragraph{Final SDE form.}
Plugging this into Eq.~\eqref{eq:app_reverse}:
\begin{equation}
\label{eq:app_sde_final}
\begin{split}
dx_j &= v_\theta(x_j, j) dj \\
&\quad + \tfrac{\sigma_j^2}{2j} \big(x_j + (1-j)v_\theta(x_j, j)\big) dj \\
&\quad + \sigma_j dw,
\end{split}
\end{equation}
which is Eq.~\eqref{eq:fm_sde} in the main text.

\subsection{Euler-Maruyama Discretization}

Discretizing Eq.~\eqref{eq:app_sde_final} with time step $\Delta j = j_{t-1} - j_t < 0$:
\begin{equation}
\label{eq:app_em}
\begin{split}
x_{j_{t-1}} &= x_{j_t} + \big[ v_\theta(x_{j_t}, j_t) \\
&\quad + \tfrac{\sigma_{j_t}^2}{2j_t}(x_{j_t} + (1-j_t)v_\theta(x_{j_t}, j_t)) \big] \Delta j \\
&\quad + \sigma_{j_t}\sqrt{|\Delta j|} \, \epsilon,
\end{split}
\end{equation}
where $\epsilon \sim \mathcal{N}(0,I)$. This yields the Gaussian transition in Eq.~\eqref{eq:gauss_transition}:
\begin{equation}
\pi_\theta(x_{j_{t-1}} | x_{j_t}, c) = \mathcal{N}(x_{j_{t-1}}; \mu_\theta, \Sigma_{j_t}),
\end{equation}
with
\begin{align}
\mu_\theta &= x_{j_t} + \big[ v_\theta(x_{j_t}, j_t) \notag \\
&\quad + \tfrac{\sigma_{j_t}^2}{2j_t}(x_{j_t} + (1-j_t)v_\theta(x_{j_t}, j_t)) \big] \Delta j,\label{eq:app_mean} \\
\Sigma_{j_t} &= \sigma_{j_t}^2 |\Delta j| \cdot I. \label{eq:app_covariance}
\end{align}

Following \citet{liu2025flowgrpotrainingflowmatching}, we use 
$\sigma_j = a\sqrt{j/(1-j)}$ with $a=0.7$. The Gaussian form enables tractable computation of log-probabilities and likelihood ratios used by GRPO (formulated below).

\subsection{GRPO Objective for Flow Matching}
\label{app:grpo_details}

Building on the Gaussian transition in Eq.~\eqref{eq:gauss_transition}, 
GRPO optimizes:
\begin{equation}
\label{eq:grpo_objective}
\begin{split}
\mathcal{J}_{\text{GRPO}}(\theta) = \mathbb{E}_{c \sim \mathcal{D}} \bigg[ &\frac{1}{G} \sum_{i=1}^G \bigg( \frac{1}{T} \sum_{t=0}^{T-1} \\
&\min\big( r_t^{(i)} \hat{A}^{(i)}, \\
&\quad \text{clip}(r_t^{(i)}, 1-\epsilon, 1+\epsilon) \hat{A}^{(i)} \big) \\
&- \beta D_{\text{KL}}(\pi_\theta \| \pi_{\theta_{\text{old}}}) \bigg) \bigg],
\end{split}
\end{equation}
where $r_t^{(i)}$ is the likelihood ratio (Eq.~\ref{eq:ratio}), 
$\hat{A}^{(i)}$ is the group-relative advantage, and $D_{\text{KL}}$ 
is the trajectory-level KL divergence (see Appendix~\ref{app:kl_implementation} 
for analytical computation).

\clearpage
\section{KL Divergence Computation}
\label{app:kl_implementation}

This appendix provides (i) the analytical form of the KL divergence between Gaussian policies used in our GRPO objective, and (ii) explicit expressions for the discretization scheme referenced in Section~\ref{subsec:preliminaries}.

\subsection{KL Divergence for Gaussian Policies}

Given two Gaussian transition policies $\pi_\theta$ and $\pi_{\text{ref}}$ with distributions
\begin{align}
\pi_\theta(x_{j_{t-1}} | x_{j_t}, c) &= \mathcal{N}(x_{j_{t-1}}; \mu_\theta, \Sigma_{j_t}), \\
\pi_{\text{ref}}(x_{j_{t-1}} | x_{j_t}, c) &= \mathcal{N}(x_{j_{t-1}}; \mu_{\text{ref}}, \Sigma_{j_t}),
\end{align}
the per-step KL divergence has a closed form. Since both policies share the same covariance $\Sigma_{j_t} = \sigma_{j_t}^2 |\Delta j| \cdot I$ (as derived in Appendix~\ref{app:sde_derivation}), the KL divergence simplifies to:
\begin{equation}
\label{eq:app_kl_step}
\begin{split}
D_{\text{KL}}(\pi_\theta \| \pi_{\text{ref}}) 
&= \mathbb{E}_{x_{j_{t-1}} \sim \pi_\theta} \left[ \log \frac{\pi_\theta}{\pi_{\text{ref}}} \right] \\
&= \frac{1}{2\sigma_{j_t}^2 |\Delta j|} \| \mu_\theta - \mu_{\text{ref}} \|^2.
\end{split}
\end{equation}

Substituting the mean expressions from Eq.~\eqref{eq:app_mean} in Appendix~\ref{app:sde_derivation}:
\begin{equation}
\label{eq:app_kl_velocity}
\begin{split}
D_{\text{KL}}(\pi_\theta \| \pi_{\text{ref}}) 
&= \frac{(\Delta j)^2}{2\sigma_{j_t}^2 |\Delta j|} \Big\| v_\theta(x_{j_t}, j_t) \\
&\quad + \tfrac{\sigma_{j_t}^2}{2j_t}\big(x_{j_t} + (1-j_t)v_\theta\big) \\
&\quad - v_{\text{ref}}(x_{j_t}, j_t) \\
&\quad - \tfrac{\sigma_{j_t}^2}{2j_t}\big(x_{j_t} + (1-j_t)v_{\text{ref}}\big) \Big\|^2 \\
&= \frac{|\Delta j|}{2\sigma_{j_t}^2} \Big\| \big(1 + \tfrac{\sigma_{j_t}^2 (1-j_t)}{2j_t}\big) \\
&\quad \times (v_\theta - v_{\text{ref}}) \Big\|^2.
\end{split}
\end{equation}

For our noise schedule $\sigma_j = a\sqrt{j/(1-j)}$ with $a=0.7$, and uniform time grid $\Delta j = -1/T$, Eq.~\eqref{eq:app_kl_velocity} provides an efficient way to compute the KL penalty in the GRPO objective without expensive Monte Carlo estimation.

\paragraph{Trajectory-level KL.}
The trajectory-level KL divergence is:
\begin{align}
&D_{\text{KL}}(\pi_\theta(\tau|c) \| \pi_{\text{ref}}(\tau|c)) \notag \\
&= \sum_{t=0}^{T-1} D_{\text{KL}}(\pi_\theta(x_{j_{t-1}}|x_{j_t},c) 
  \| \pi_{\text{ref}}(x_{j_{t-1}}|x_{j_t},c)).
\end{align}

\subsection{Discretization Scheme Details}

As mentioned in Section~\ref{subsec:preliminaries}, we discretize the SDE (Eq.~\ref{eq:fm_sde}) using the Euler-Maruyama method. The mean and covariance of the resulting Gaussian transition (Eq.~\ref{eq:gauss_transition}) are given by Eqs.~\eqref{eq:app_mean} and \eqref{eq:app_covariance} in Appendix~\ref{app:sde_derivation}.

\paragraph{Time grid.} 
We use a uniform grid with $T$ steps: $j_t = t/T$ for $t = 0, 1, \ldots, T$, yielding constant step size $\Delta j = j_{t-1} - j_t = -1/T$ (negative for the reverse process).

\paragraph{Log-probability computation.}
The per-step log-probability is:
\begin{equation}
\begin{split}
\log \pi_\theta(x_{j_{t-1}} | x_{j_t}, c) 
&= -\tfrac{d}{2}\log(2\pi \sigma_{j_t}^2 |\Delta j|) \\
&\quad - \tfrac{\| x_{j_{t-1}} - \mu_\theta \|^2}{2\sigma_{j_t}^2 |\Delta j|},
\end{split}
\end{equation}
where $d = H \times W \times 3$ is the image dimensionality and $\mu_\theta \equiv \mu_\theta(x_{j_t}, j_t, c)$ from Eq.~\eqref{eq:app_mean}.

These closed-form expressions enable efficient computation of likelihood ratios (Eq.~\ref{eq:ratio}) and KL divergence in our GRPO-based optimization.

\section{Stability Analysis of APEX}
\label{app:theory}

Global convergence to Pareto optimality under stochastic gradients and 
dynamic weighting remains an open question. Following the analysis framework 
of \citet{lu2025learningoptimizemultiobjectivealignment}, we prove that 
the $\mathcal{P}^3$ weight update mechanism maintains numerical stability 
through bounded priority scores and softmax normalization: weights remain 
bounded and update smoothly without collapse or explosion.

\subsection{Why Multiplicative Aggregation}
\label{app:multiplication}

The $\mathcal{P}^3$ mechanism combines three factors via multiplication 
(Eq.~\ref{eq:p3_priority}):
\begin{equation}
\Psi_k = \mathcal{P}^{\text{learn}}_k \cdot \mathcal{P}^{\text{conflict}}_k 
\cdot \mathcal{P}^{\text{need}}_k.
\end{equation}
This design follows the geometric aggregation principle in multi-criteria 
optimization~\citep{Marler2004}, offering several advantages over 
additive formulations:

\paragraph{Non-compensatory aggregation.}
Additive combination $\Psi_k = \lambda_1 P^{\text{learn}}_k + \lambda_2 P^{\text{conflict}}_k + \lambda_3 P^{\text{need}}_k$ 
allows a high score in one factor to compensate for low scores in others. 
Multiplicative aggregation prevents this: low performance in any factor 
yields low overall priority, ensuring objectives are prioritized only when 
all three conditions (high gradient, low conflict, room for improvement) 
are simultaneously met.

\paragraph{Parameter-Free Combination.}
All factors are normalized to comparable ranges (Lemma~\ref{lem:bounded_factors}), 
so their product directly yields a composite score without requiring 
tuning of mixing coefficients $\{\lambda_1,\lambda_2,\lambda_3\}$.

\subsection{Assumptions}

We adopt standard assumptions from the RL and multi-objective 
optimization literature.

\begin{assumption}[Bounded Policy Gradients]
\label{ass:bounded_grad}
For each objective $k \in \{1,\dots,K\}$, the policy gradient is bounded: 
$\|\nabla_\theta J_k(\theta)\| \leq M < \infty$.
\end{assumption}

\begin{assumption}[Bounded Rewards]
\label{ass:bounded_reward}
All reward functions are bounded: $R_k(x_0, c) \in [0, R_{\max}]$ 
for all $k, x_0, c$.
\end{assumption}

\begin{assumption}[Empirical Utopia Points]
\label{ass:utopia}
The Utopia Point $U_k$ is set to the empirical upper bound achieved by 
single-objective specialist models (Section~4.1). The running performance 
estimate $\bar{R}_k^{(t)}$ tracks the model's current capability on 
objective $k$ during multi-objective training, typically remaining below $U_k$.
\end{assumption}

Assumptions~\ref{ass:bounded_grad} and~\ref{ass:bounded_reward} are 
standard in policy gradient analysis~\citep{wang2017finite,kumar2023sample}. 
In our setting, Assumption~\ref{ass:bounded_reward} holds because rewards 
are either normalized (OCR $\in$ [0,1]) or inherently bounded (PickScore, DeQA, Aesthetic). 
For Assumption~\ref{ass:utopia}, we set $U_k$ to single-objective 
specialist performance (Section~4.1).

\subsection{Bounded Priority Factors}

\begin{lemma}[Bounded Priority Factors]
\label{lem:bounded_factors}
Under Assumptions~\ref{ass:bounded_grad}, \ref{ass:bounded_reward}, 
and~\ref{ass:utopia}, the priority factors satisfy:
\begin{align}
\mathcal{P}^{\text{learn}}_k &\in [0, 1], \\
\mathcal{P}^{\text{conflict}}_k &\in [0, 1], \\
\mathcal{P}^{\text{need}}_k &\in [1, 2].
\end{align}
\end{lemma}

\begin{proof}
(i) By definition (Eq.~\ref{eq:p3_learn}), 
$\mathcal{P}^{\text{learn}}_k$ is a normalized fraction, hence $\in [0, 1]$.

(ii) From Eq.~\ref{eq:p3_conflict}:
\begin{equation}
\mathcal{P}^{\text{conflict}}_k = 1 + \frac{1}{K-1}\sum_{\ell\neq k}
\min(0, \cos\phi_{k,\ell}).
\end{equation}
Since $\min(0, \cos\phi) \in [-1, 0]$, the sum ranges in $[-(K-1), 0]$, 
giving $\mathcal{P}^{\text{conflict}}_k \in [0, 1]$.

(iii) From Eq.~\ref{eq:p3_need} with 
$U_k > \bar{R}_k^{(t)} \geq 0$ (Assumptions~\ref{ass:bounded_reward} 
and~\ref{ass:utopia}):
\begin{equation}
\begin{split}
\mathcal{P}^{\text{need}}_k &= 1 + \max\left(0, \frac{U_k - \bar{R}_k^{(t)}}{U_k + \epsilon}\right) \\
&\in \left[1, 1 + \frac{U_k}{U_k+\epsilon}\right) \subseteq [1, 2).
\end{split}
\end{equation}
\end{proof}

\begin{corollary}[Bounded Composite Priority]
The composite priority score satisfies 
$\Psi_k^{(t)} = \mathcal{P}^{\text{learn}}_k \cdot \mathcal{P}^{\text{conflict}}_k 
\cdot \mathcal{P}^{\text{need}}_k \in [0, 2]$.
\end{corollary}

\subsection{Weight Stability Guarantees}

\begin{theorem}[APEX Weight Stability]
\label{thm:apex_stability}
Under the $\mathcal{P}^3$ update rule (Eq.~\ref{eq:p3_priority}) with bounded priorities 
(Lemma~\ref{lem:bounded_factors}) and $\tau > 0$, for any training step 
$t \geq 0$, the weights satisfy:
\begin{itemize}
\item[(i)] \textbf{Simplex preservation}: $\mathbf{w}^{(t)} \in \Delta^{K-1}$ 
      for all $t$.
\item[(ii)] \textbf{Bounded weight ratio}: 
      \begin{equation}
      \frac{w_i^{(t)}}{w_j^{(t)}} \leq \exp(2/\tau).
      \end{equation}
      When $\tau=1$, the ratio is bounded by $\exp(2) \approx 7.39$.
\item[(iii)] \textbf{Smooth updates}: 
      \begin{equation}
      \|\mathbf{w}^{(t+1)} - \mathbf{w}^{(t)}\|_1 \leq 2(1 - \exp(-2/\tau)).
      \end{equation}
\end{itemize}
\end{theorem}

\begin{proof}
(i) Follows from softmax normalization (Eq.~\ref{eq:p3_priority}).

(ii) From softmax:
\begin{equation}
\frac{w_i^{(t)}}{w_j^{(t)}} = \exp\left(\frac{\Psi_i^{(t)} - \Psi_j^{(t)}}{\tau}\right) 
\leq \exp(2/\tau),
\end{equation}
using $\Psi_k \in [0, 2]$ (Corollary).

(iii) By Lipschitz continuity of softmax, with maximum priority 
change $\delta=2$.
\end{proof}

\subsection{Discussion}

\paragraph{What we prove.}

Theorem~\ref{thm:apex_stability} establishes three properties. 
First, softmax normalization prevents weight collapse—unlike 
multiplicative updates~\citep{lu2025learningoptimizemultiobjectivealignment} 
that require convergent learning rate schedules. 
Second, temperature $\tau$ controls adaptation rate: smaller $\tau$ 
enables aggressive updates, larger $\tau$ smooths changes. 
Third, the weight ratio bound $\exp(2/\tau)$ is independent of training 
history, unlike cumulative bounds in prior work.

\paragraph{DSAN scale invariance.}
Stage~1 z-score normalization removes scale; Stage~2 operates on 
normalized values. This ensures that for any scaling $\{c_k > 0\}$, 
replacing $R_k \to c_k R_k$ leaves $\hat{A}_{\text{final}}^{(i)}$ 
unchanged, preventing variance hijacking (Section~3.2).

\paragraph{What we do not prove.}
We do not establish convergence to Pareto optimality (which requires 
strong convexity and exact gradients), regret bounds (dynamic weights 
create non-stationary MDPs), or long-horizon weight convergence 
(though empirically observed in Appendix~\ref{app:dynamics}). 
Our guarantees are numerical: weights remain bounded and stable. 
Empirical results (Sections~\ref{sec:main}--\ref{sec:analysis}) 
confirm this translates to effective Pareto approximation.

\section{Details of Experiments}
\label{app:experiments}

This appendix provides comprehensive experimental details, additional 
qualitative examples, and in-depth analysis that complement the main 
results in Section~4.

\subsection{Implementation Details}
\label{app:implementation}

\textbf{Training Data.} 
To ensure the OCR objective receives adequate training signal, we use the OCR training 
prompts from Flow-GRPO~\citep{liu2025flowgrpotrainingflowmatching}, which contain 20K 
examples following the template ``A sign that says "[text]"'', where text enclosed in 
double quotes specifies the exact string to be rendered in the generated image. Text 
strings are generated by GPT-4o and span diverse categories including common phrases, 
brand names, warnings, and creative text, with lengths ranging from 2 to 20 
characters. During multi-objective training, all four reward objectives (OCR, PickScore, 
DeQA, Aesthetic) are optimized jointly on these prompts. Evaluation uses 1K held-out 
OCR prompts for text rendering accuracy and the DrawBench~\citep{NEURIPS2022_ec795aea} for image quality metrics.

\paragraph{Hardware and Optimization.}
Training is performed on 8 NVIDIA A100 GPUs for a total of 1,200 steps (main experiments). 
We use the AdamW optimizer with learning rate $3 \times 10^{-4}$ and 
default momentum parameters ($\beta_1=0.9$, $\beta_2=0.999$). The model 
is fine-tuned using LoRA~\citep{hu2021loralowrankadaptationlarge} with 
rank $r=32$ and scaling factor $\alpha=64$, applied to the MM-DiT 
architecture of SD3.5-Medium. Training terminates at 1,200 steps, at 
which point performance approaches single-objective specialist levels 
(Table~\ref{tab:main_results}). Extended training may yield further 
gains but is computationally prohibitive given our resource constraints.

\paragraph{Sampling Strategy.}
During training, we sample 48 unique prompts per epoch, generating $G=24$ 
independent rollouts per prompt under the stochastic policy. To improve 
sample efficiency, we adopt a \textbf{denoising reduction} strategy: 
training rollouts use 10 denoising steps (with SDE noise schedule 
$\sigma_j = 0.7\sqrt{j/(1-j)}$), while final inference uses 40 steps 
to ensure high-quality generation. This reduces data collection cost by 
$\sim$4× without degrading final performance.

\paragraph{Gradient Estimation for $\bm{\mathcal{P}^3}$.}
As referenced in Section 3.3.2, to minimize computational overhead, gradient information 
$\{\nabla_\theta J_k\}_{k=1}^K$ for the $\bm{\mathcal{P}^3}$ mechanism is estimated from 
a single micro-batch of size 8 per epoch (incurring $<10\%$ additional 
time cost). Raw gradient estimates are smoothed using an Exponential 
Moving Average (EMA) with decay rate $\gamma=0.8$:
\begin{equation}
\nabla J_k^{(t)} \leftarrow \gamma \nabla J_k^{(t-1)} + (1-\gamma) \hat{\nabla} J_k^{(t)},
\end{equation}
where $\hat{\nabla} J_k^{(t)}$ is the current mini-batch estimate.

\paragraph{Hyperparameter Settings.}
All APEX coefficients are set to 1.0: temperature parameter $\tau=1$ 
for the softmax in Eq.~\ref{eq:p3_priority}. The KL regularization coefficient in GRPO 
is set to $\beta=0.01$ for all experiments.

\paragraph{Utopia Points.}
The reference upper bounds $U_k$ for Progress Need (Eq.~\ref{eq:p3_need}) 
are determined as follows: for objectives with available single-objective 
specialists (Table~\ref{tab:main_results}), we use their normalized performance; 
for others, we set empirically estimated upper bounds based on observed score 
distributions. Specifically: OCR $= 0.92$ (from OCR-Only specialist), 
PickScore $= 0.90$ (from PickScore-Only specialist), 
DeQA $= 0.86$ (empirical estimate), Aesthetic $= 0.62$ (empirical estimate). 
These bounds guide the Progress Need factor without requiring exhaustive 
single-objective training for all four objectives.

\paragraph{Running Performance Tracking.}
The performance estimate $\bar{R}_k^{(t)}$ is computed as the mean reward 
over the current training batch:
\begin{equation}
\bar{R}_k^{(t)} = \frac{1}{B \cdot G}\sum_{b=1}^B\sum_{i=1}^G R_k(x_0^{(b,i)}, c_b),
\end{equation}
where $B$ is the number of prompts per batch and $G=24$ is the group size. 
This batch-level average provides an instantaneous estimate of the policy's 
current capability on objective $k$, used by the Progress Need factor 
(Eq.~\ref{eq:p3_need}) to identify bottleneck objectives.

\subsection{Hypervolume Computation Details}
\label{app:hypervolume}

We employ the Hypervolume (HV) indicator \citep{zitzler1999multiobjective} 
to quantify Pareto front approximation quality across three experimental contexts. 

\paragraph{Pareto Domination.}
For two solution vectors $a, b \in \mathbb{R}^K$ in a maximization setting, 
$a$ \textbf{dominates} $b$ (denoted $a \succ b$) if and only if 
$a_j \geq b_j$ for all $j \in \{1, \ldots, K\}$ and 
$a_j > b_j$ for at least one $j$. 
A solution $a$ is \textbf{non-dominated} (Pareto-optimal) in a set 
$\mathcal{A}$ if no other solution in $\mathcal{A}$ dominates it.

\paragraph{Hypervolume Definition.}
For a finite set $\mathcal{A} = \{a^{(1)}, \ldots, a^{(n)}\} \subset \mathbb{R}^K$ 
and reference point $\mathbf{r} = (r_1, \ldots, r_K)$, hypervolume is defined as:
\begin{equation}
\label{eq:hypervolume_def}
\text{HV}(\mathcal{A}; \mathbf{r}) = \Lambda\left(\bigcup_{a \in \mathcal{A}} 
[a, \mathbf{r}]\right),
\end{equation}
where $[a, \mathbf{r}] = \{x \in \mathbb{R}^K \mid r_j \leq x_j \leq a_j, 
\forall j \in \{1, \ldots, K\}\}$ and $\Lambda$ denotes the $K$-dimensional 
Lebesgue measure. Intuitively, HV measures the volume of objective space 
\textbf{dominated by} $\mathcal{A}$ (i.e., the union of all hyperrectangles 
$[a, \mathbf{r}]$ for $a \in \mathcal{A}$) and bounded below by $\mathbf{r}$; 
larger values indicate better Pareto front coverage. We compute HV using the 
dimension-sweep algorithm \citep{fonseca2006improved}, which achieves 
$O(n^{K-2} \log n)$ complexity for $n$ points in $K$ dimensions.

\paragraph{Metric Normalization.}
To ensure commensurability across heterogeneous objectives, we apply simple 
range normalization before HV computation. OCR accuracy is already in $[0,1]$ 
and requires no scaling. For other metrics, we divide by empirical upper bounds: 
PickScore$/26$ (typical maximum), DeQA$/5$ (defined range), and Aesthetic$/10$ 
(conservative upper bound to accommodate occasional high-scoring outliers). 
This normalization does not affect training dynamics, as DSAN (Section~\ref{sec:dsan}) 
performs additional standardization within GRPO updates.

\paragraph{Reference Points.}
We employ three reference configurations depending on the evaluation context. 
\textbf{All references follow the canonical (OCR, PickScore, DeQA, Aesthetic) 
ordering used in Table~\ref{tab:main_results}.}

\begin{itemize}[leftmargin=*,noitemsep,topsep=3pt]
    \item \textbf{Main Results (Table~\ref{tab:main_results}):} 
    Reference point is set to base model (SD3.5-M) performance. 
    Original metric values: $\mathbf{r}_{\text{base}}^{\text{raw}} = 
    (0.59, 21.72, 4.07, 5.39)$ for (OCR, PickScore, DeQA, Aesthetic) 
    respectively. After normalization (OCR unchanged, PickScore$/26$, 
    DeQA$/5$, Aesthetic$/10$), the reference becomes 
    $\mathbf{r}_{\text{base}} = (0.59, 0.835, 0.814, 0.539)$. 
    Since only the final checkpoint is evaluated per model, HV serves 
    as a scalar proxy for overall improvement magnitude rather than 
    Pareto set diversity.
    
    \item \textbf{Training Dynamics Analysis (Section~\ref{sec:analysis}, 
    Figure~\ref{fig:analysis}(d)):} Reference point reflects early-training 
    performance (normalized coordinates): $\mathbf{r}_{\text{early}} = (0.38, 0.81, 0.60, 0.50)$ . 
    We evaluate 10 evenly-spaced checkpoints 
    (steps $\{120, 240, \ldots, 1200\}$) on a held-out test set, 
    computing HV over the non-dominated subset to track Pareto front evolution.
    
    \item \textbf{Ablation Studies (Table~\ref{tab:p3_ablation}):} 
    Same reference as training dynamics: 
    $\mathbf{r}_{\text{early}} = (0.38, 0.81, 0.60, 0.50)$. 
    HV is computed over training batch rewards using non-overlapping 
    50-step windows, averaging across multiple epochs to reduce stochastic 
    variance from rollout sampling.
\end{itemize}

\paragraph{Rationale for Context-Specific References.}
The choice of reference point affects hypervolume interpretation 
\citep{fonseca1996performance}. For Table~\ref{tab:main_results}, using 
base model performance as reference enables intuitive comparison of 
post-training gains. For temporal analyses (Figure~\ref{fig:analysis}(d), 
Table~\ref{tab:p3_ablation}), the early-training reference avoids artificial 
inflation from large initial improvements, better highlighting incremental 
Pareto front expansion during fine-tuning. We use test-set evaluation in 
Figure~\ref{fig:analysis}(d) to rigorously validate Pareto front evolution, 
while training-batch evaluation in Table~\ref{tab:p3_ablation} enables 
efficient ablation across multiple $\mathcal{P}^3$ configurations.

\begin{figure*}[t]
  \centering
  \includegraphics[width=\textwidth]{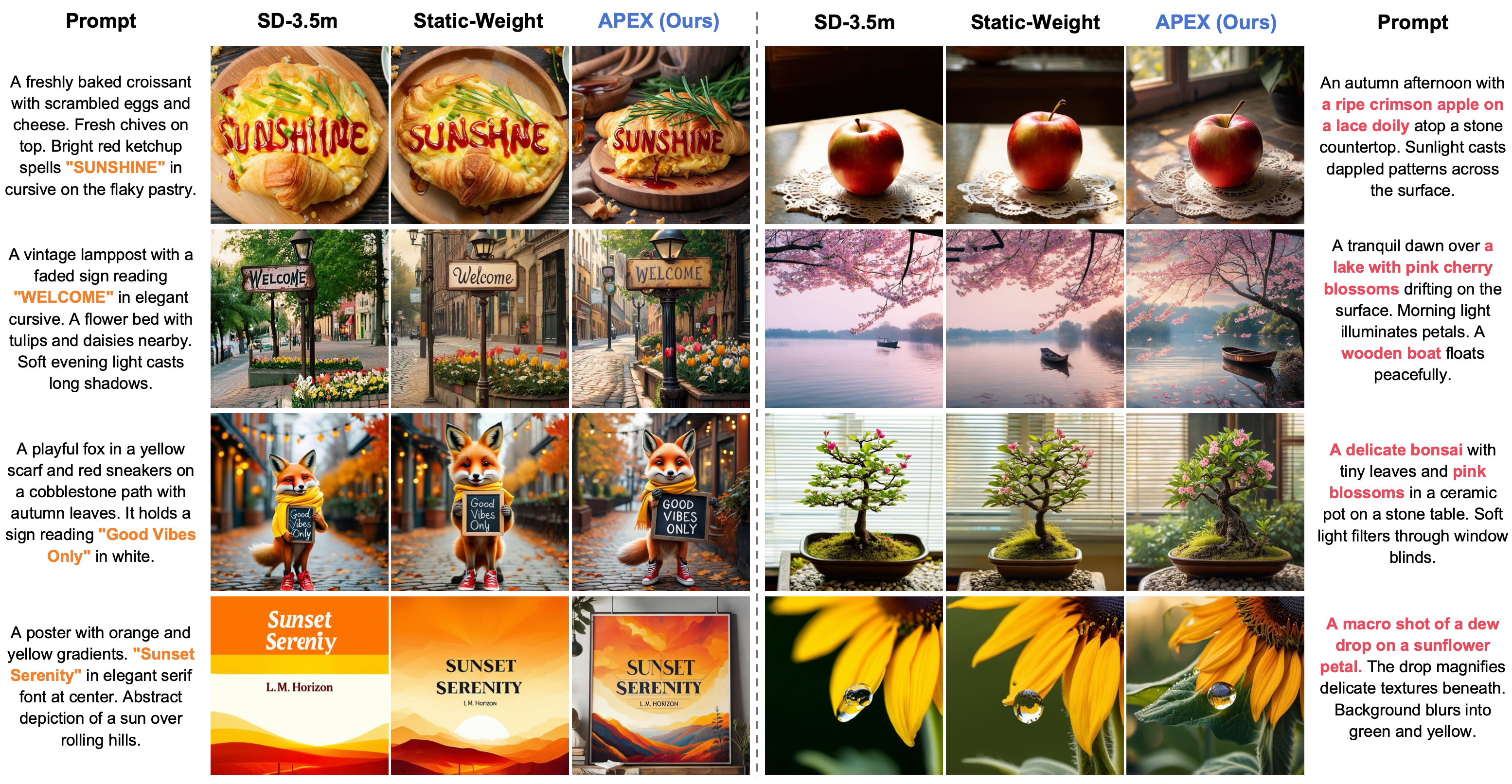}
  \caption{\textbf{Qualitative comparison across text-centric (left) and aesthetics-centric (right) scenarios.}
  In text-intensive generation, APEX achieves superior spelling accuracy and semantic coherence while maintaining visual realism. In photorealistic generation, APEX demonstrates enhanced detail preservation across lighting, textures, and atmospheric depth, exploring more favorable Pareto frontiers. Both baselines exhibit various failure modes including text errors, unrealistic rendering, and limited aesthetic improvement.}
  \label{fig:comparison}
  \vspace{-3mm} 
\end{figure*}

\subsection{Qualitative Comparison}
\label{app:qualitative}

Figure~\ref{fig:comparison} presents a visual comparison between APEX and the baseline methods across diverse scenarios. We evaluate the results from two perspectives: instruction alignment accuracy and visual detail fidelity. 

\textbf{Instruction Alignment Accuracy.} Although the Static-Weight method tends to prioritize local text features to achieve higher OCR scores in Table \ref{tab:main_results}, APEX demonstrates superior comprehensive capability in coordinating text, semantics, and visuals. In tasks involving complex text rendering, APEX effectively addresses typical failure modes of baseline methods. Specifically, SD-3.5m and Static-Weight exhibit spelling errors or character deformations in Row 1 and Row 3, whereas APEX accurately renders the complete text while maintaining natural integration with the object's texture. In Row 2, although Static-Weight generates clear fonts, the overly pristine texture of the sign contradicts the ``vintage'' stylistic constraint; in contrast, APEX precisely restores the weathered texture and ambient lighting while maintaining legibility. For the poster prompt in Row 4, which contains compositional ambiguity, APEX tends to construct a 3D scene with higher spatial complexity rather than a simple 2D layout, demonstrating its ability to explore diverse optimal solutions on the Pareto front.

\textbf{Visual Detail Fidelity.} In photorealistic rendering, APEX consistently outperforms the baselines in lighting modeling, color interaction, and physical logic, whereas Static-Weight shows limited improvement over SD-3.5m in these dimensions. A representative case is the cherry blossom lake (Row 6), where APEX achieves precise color isolation, preventing the background vegetation from being affected by ``color-bleeding'' from the pink blossoms, while capturing realistic reflections and atmospheric depth. Similar detail enhancements are evident in the delicate lace lighting (Row 5), the realistic bark textures and blind-filtered light (Row 7), and the physically-grounded dewdrop refraction and natural droplet distribution (Row 8). These observations suggest that APEX's adaptive priority mechanism more effectively mines fine-grained, aesthetics-related reward signals.

The qualitative comparison corroborates the quantitative findings in Table \ref{tab:main_results}: APEX effectively mitigates the ``variance hijacking'' phenomenon in multi-objective alignment. While maintaining competitiveness in key text metrics, it avoids semantic misalignment caused by overfitting a single objective, achieving a systemic balance between perceptual quality and instruction following.

\subsection{Composite Advantage Variance Dynamics}
\label{app:dynamics}

To verify the necessity of the second stage of DSAN normalization, 
we track the variance of the composite advantages after weighted 
aggregation but before the second normalization (Fig.~\ref{fig:variance_decay}). 
Experimental observations show a significant and continuous decay 
in advantage variance, dropping from $\sim 0.40$ initially to 
$\sim 0.28$ in the later stages. According to the principle of 
variance decomposition, while weight magnitudes remain relatively 
stable, the drop in total variance is attributed to the objective-wise 
covariance becoming negative. This statistically confirms that as 
training progresses, the model enters a ``zero-sum game'' region 
near the Pareto front, where gradient interference between objectives 
leads to signal cancellation. Without second-stage normalization, 
this natural variance decay would be equivalent to a passive reduction 
in the learning rate. APEX's second-stage normalization provides an 
adaptive signal rescaling mechanism that forces the decayed composite 
advantages back to a standard distribution, compensating for signal 
loss caused by objective conflicts.

\subsection{Stability Analysis of \texorpdfstring{$\bm{\mathcal{P}^3}$}{P3} Factors}
\label{app:ablation}

\begin{figure}[t]
  \centering
  \includegraphics[width=\columnwidth]{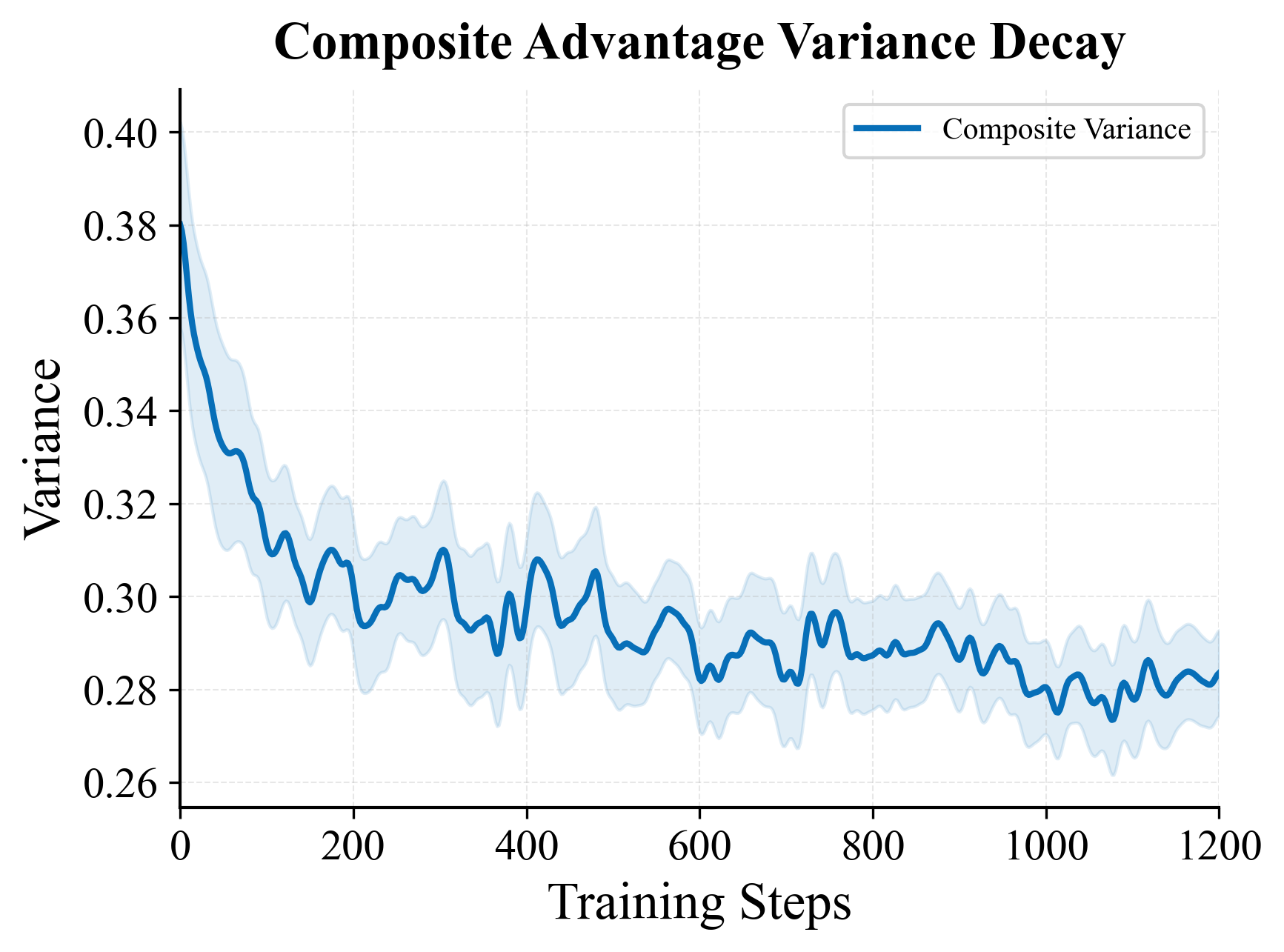}
  \caption{\textbf{Composite Advantage Variance Decay.} The weighted advantage variance decreases over training as DSAN normalizes conflicting gradients, reducing inter-objective interference.}
  \label{fig:variance_decay}
\end{figure}

\begin{figure}[!t]
  \centering
  \includegraphics[width=\columnwidth]{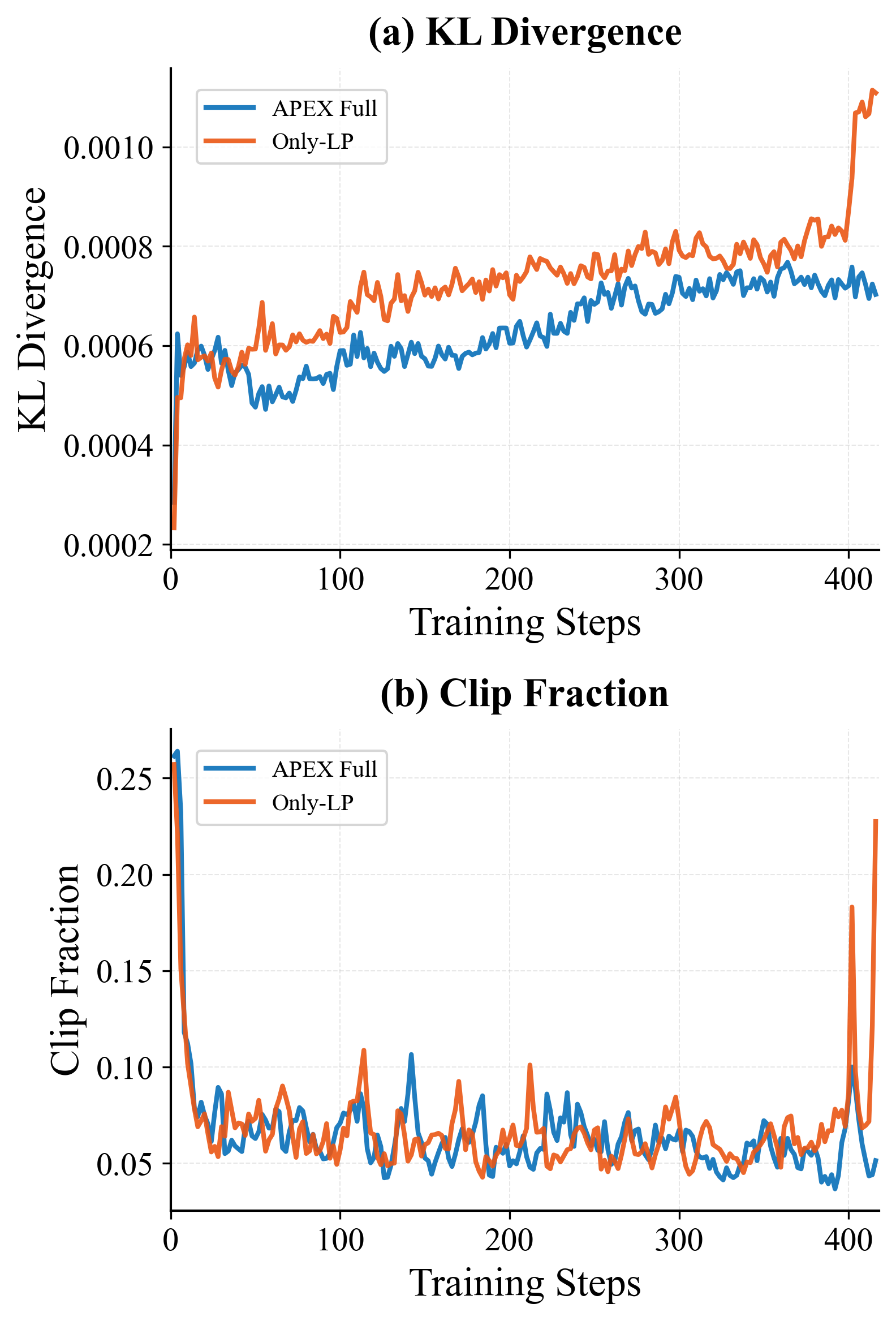}
  \caption{\textbf{Stability analysis of $\mathcal{P}^3$ factors.} The Only-LP variant ($\alpha{=}0, \beta{=}0$) shows sudden spikes in KL divergence and clipping fraction, indicating unstable policy updates. APEX Full remains stable, confirming that CP and PN factors act as safety valves against gradient conflicts and objective saturation.}
  \label{fig:stability}
\end{figure}

We reduce APEX to the \textbf{Only-LP} variant ($\alpha{=}0, \beta{=}0$), 
retaining only the learning potential factor. As shown in 
Fig.~\ref{fig:stability}, the Only-LP variant performs similarly to 
APEX Full in the early stages, but as training deepens, its KL divergence 
and clipping fraction exhibit sudden, sharp spikes. This indicates that 
solely pursuing high gradient magnitude leads to overly aggressive policy 
updates. Without the constraints of Conflict Penalty (CP), the model 
forces updates when gradient directions are inconsistent, triggering 
parameter oscillations. Without Progress Need (PN) regulation, the model 
continues to apply high weights after certain objectives are saturated, 
causing the policy to deviate severely from the reference model. In 
contrast, APEX Full maintains a stable training trajectory throughout. 
This proves that the CP and PN factors serve as critical \textbf{Safety 
Valves} within the framework, ensuring convergence stability by dynamically 
suppressing excessive weight allocation during conflicts or objective 
saturation.

\end{document}